%% file: mfm_paper.tex
\documentclass[12pt]{article}
\usepackage{fullpage,graphicx,amsmath,amssymb,verbatim,xcolor}
\usepackage[small,bf]{caption}

\usepackage[noend]{algpseudocode}
\usepackage{algorithm}
\usepackage{url}
\usepackage{tikz}
\usetikzlibrary{intersections}

\usepackage{subcaption}

\usetikzlibrary{positioning,calc,shapes,arrows}
\tikzstyle{block} = [rectangle, draw,
text width=7em, text centered, minimum height=5em]
\tikzstyle{line} = [draw, -latex',line width=1mm]

\usepackage{hyperref}
\hypersetup{
colorlinks   = true,    % Colours links instead of ugly boxes
urlcolor     = blue,    % Colour for external hyperlinks
linkcolor    = blue,    % Colour of internal links
citecolor    = red      % Colour of citations
}

\usepackage{cleveref}
\crefname{equation}{}{}
\Crefname{equation}{}{}
\input defs.tex

\usepackage{amsthm}
\theoremstyle{remark} 
\newtheorem{remark}{Remark}

\newcommand{\diag}{\mathop{\bf diag}}
\newtheorem{lemma}{Lemma}[section]

\title{Fitting Multilevel Factor Models}

\author{Tetiana Parshakova \and Trevor Hastie \and Stephen Boyd}

\begin{document}
\maketitle

\begin{abstract}
We examine a special case of the multilevel factor model, 
with covariance given by multilevel low rank (MLR) matrix~\cite{parshakova2023factor}. 
We develop a novel, fast implementation of the expectation-maximization
algorithm, tailored for multilevel factor models, to maximize the 
likelihood of the observed data.
This method accommodates any hierarchical 
structure and maintains linear time and storage complexities per 
iteration. 
This is achieved through a new efficient technique for computing the inverse 
of the positive definite MLR matrix.
We show that the inverse of positive definite MLR matrix is also an 
MLR matrix with the same sparsity in factors, 
and we use the recursive Sherman-Morrison-Woodbury 
matrix identity to obtain the factors of the inverse.
Additionally, we present an algorithm that computes the Cholesky factorization of 
an expanded matrix with linear time and space complexities, 
yielding the covariance matrix as its Schur complement.
This paper is accompanied by an open-source package that implements the
proposed methods.
\end{abstract}

\clearpage
\tableofcontents
\clearpage

\section{Introduction}
Factor models are used to explain the variation in the observed 
variables through a smaller number of factors.
In fields like biology, economics, and social sciences, the 
data often has hierarchical structures.
To capture this structure specialized multilevel factor models were
developed.
Existing methods for fitting these models do not scale well with large 
datasets.

In this work, we introduce an efficient algorithm for fitting multilevel factor models. 
Our method is compatible with any 
hierarchical structure and achieves linear time and storage 
complexity per iteration. 

\subsection{Prior work}
\paragraph{Factor models.}
Factor analysis was initially developed to address problems in psychometrics
about 120 years ago~\cite{10.2307/1412107},
and it later found applications in psychology, finance, economics, and statistics.
The idea behind factor analysis is to describe variability among the observed
variables using a small number of unobserved variables called factors.
Factor models decompose a covariance matrix into a sum of a low rank matrix, 
associated with underlying factors, and a diagonal matrix, 
representing idiosyncratic variances.
Since the early 20th century, 
factor analysis has seen significant methodological 
advancements~\cite{fruchter1954introduction, cattell1965biometrics, 
joreskog1969general,fama1993common, fabrigar1999evaluating},
with several books dedicated to its theory and 
application~\cite{harman1976modern, child2006essentials}.

\paragraph{Hierarchically structured data.}
Data from fields such as biology, economics, social sciences, and medical sciences 
often exhibits a hierarchical, nested, or clustered structure. 
This has led to the development of specialized techniques in factor analysis
aimed specifically at 
handling hierarchically structured data such as
hierarchical factor models~\cite{schmid1957development, wherry1959hierarchical}
and multilevel factor models~\cite{aitkin1981statistical, mcdonald1989balanced}.

\paragraph{Hierarchical factor models.}
In hierarchical factor models, factors are organized into a hierarchy, 
where general factors at the top influence more specific factors positioned 
beneath them \cite{schmid1957development, brunner2012tutorial, yung1999relationship,
raudenbush2002hierarchical}. 
This model type does not necessarily reflect a hierarchy in the data 
(\eg, individuals within groups) but rather in the latent variables themselves.
Widely used in psychometrics, these models are crucial for distinguishing 
between higher-order and lower-order factors~\cite{carroll1993human,mcgrew2009chc}.
For instance, \cite{deyoung2006higher} identified a hierarchical structure of 
personality with two general factors, stability and plasticity, 
at the top, and 
the so-called Big Five personality factors below them:
neuroticism, agreeableness, and conscientiousness are under stability, 
while extraversion and openness are under plasticity.

\paragraph{Multilevel factor models.}
Multilevel factor models are statistical frameworks developed in the
1980s to handle hierarchical data structures;
see~\cite{aitkin1981statistical, goldstein1986multilevel,
mcdonald1989balanced, rowe1998modeling, rabe2004generalized}, and 
the books~\cite{de2008handbook,goldstein2011multilevel}.
These models partition factors into global and local
components,
allowing the decomposition of the variances of observed variables into components 
attributable to each level of the hierarchy.
There is a wide variety of multilevel factor models discussed in the
literature, with the general
form for a $2$-level factor model presented 
in~\cite[\S8.2]{goldstein2011multilevel}. 

Multilevel (dynamic) factor models have also been applied to time 
series data~\cite{gregory1999common,bai2002determining,
wang2008large,bai2015identification}.
They have been particularly effective in modeling the co-movement of 
economic quantities across different 
levels~\cite{gregory1999common, bai2015identification}.
For example, \cite{kose2003international, crucini2011driving, jorg2016analyzing} 
used these models to characterize the co-movement
of international business cycles on global, regional, and country levels.

In this paper we focus on a special case of the multilevel factor model,
that has no intercept and no linear covariates.
The framework can be easily extended to more general case
as needed, see \S\ref{sec-lin-cov-fm}. 
We assume the observations follow a normal distribution,
so the model is defined by a covariance matrix that is a 
multilevel low rank (MLR) matrix~\cite{parshakova2023factor}. 
In~\cite{parshakova2023factor} authors consider two problems beyond fitting,
namely,
rank allocation and capturing partition.
Here, we assume that both rank allocation and hierarchical partition
are fixed, and focus solely on fitting factors.

\paragraph{Fitting methods.}
Several methods have been employed to fit multilevel models, each with its  
advantages and challenges. Among the most prominent are maximum likelihood and Bayesian 
estimation techniques~\cite{dedrick2009multilevel},
and Frobenius norm-based fitting methods~\cite{parshakova2023factor}. 
Commonly utilized algorithms for these methods include the expectation-maximization (EM)
algorithm~\cite{rubin1982algorithms, raudenbush1995maximum}, 
the Newton-Raphson algorithm~\cite{lindstrom1988newton}, 
iterative generalized least squares~\cite{goldstein1986multilevel}, 
the Fisher scoring algorithm, 
and Markov Chain Monte Carlo~\cite{goldstein2014multilevel}. 
Despite the efficacy of these approaches, no single method proves entirely 
satisfactory under all possible data conditions encountered in research. 
As a result, statisticians are continually developing alternative techniques to 
enhance model fitting and accuracy \cite{dedrick2009multilevel, lin2010comparison}.

\paragraph{Software packages.}
Several commercial packages offer capabilities for handling 
multilevel modeling, including LISREL~\cite{joreskog1996lisrel},
Mplus~\cite{asparouhov2006multilevel,muthen2017mplus,muthen2024mplus} and 
MLwiN~\cite{rasbash2000user}. 
The open-source packages include
lavaan~\cite{rosseel2012lavaan, huang2017conducting},
gllamm~\cite{rabe2004gllamm}.
Additional resources and software recommendations can be found 
in~\cite[\S1.7]{de2008handbook} and~\cite[\S18]{goldstein2011multilevel}.
% Mplus 2-level, 3-level, MLwiN 4-levels ?, lavaan 2-levels
These tools are primarily designed for multilevel linear models~\cite{gelman2007data}, 
and most of them
do not support the specific requirements of factor analysis within multilevel 
frameworks that 
involve an arbitrary number of levels in hierarchical structures.
Although OpenMx~\cite{boker2011openmx,pritikin2017many}, an open-source package
that implements MLE-based fitting methods, does support
multiple levels of hierarchy,
it was unable to handle our large-scale examples. 
Additionally, we found no high-quality, open-source implementations of 
MCMC-based fitting methods; thus these were not included in our comparison.

In this paper, leveraging the MLR structure of the covariance matrix,
we derive a novel fast implementation of the EM algorithm for multilevel factor
modeling that works with any
hierarchical structure and requires linear time and storage complexities
per iteration.

\subsection{Our contribution}

The main contributions of this paper are the following:
\begin{enumerate}
    \item We present a novel computationally efficient algorithm for
fitting multilevel factor models, which operates with linear time 
and storage complexities per iteration.
\item We show that the inverse of an invertible PSD MLR matrix is also an 
MLR matrix with the same sparsity in factors, 
and we use the recursive Sherman-Morrison-Woodbury 
matrix identity to obtain the factors of the inverse.
\item 
We present an algorithm that computes the Cholesky factorization of 
an expanded matrix with linear time and space complexities, 
yielding the covariance matrix as its Schur complement.
We also show that Cholesky factor has the same sparsity pattern as its inverse.
\item We provide an open-source package that implements the fitting method, 
available at \begin{quote}
\url{https://github.com/cvxgrp/multilevel_factor_model}
\end{quote}
We also provide several examples that illustrate our method.
\end{enumerate}

\section{Multilevel factor model}
In this section we review the multilevel low rank (MLR) matrix
along with notations necessary for our method.
We then present a variant of the multilevel factor model that will be 
the focus of this paper.

\subsection{Multilevel low rank matrices}

% \subsubsection{Contiguous PSD MLR}
\begin{figure}
    \begin{center}    
    \includegraphics[width=0.7\textwidth]{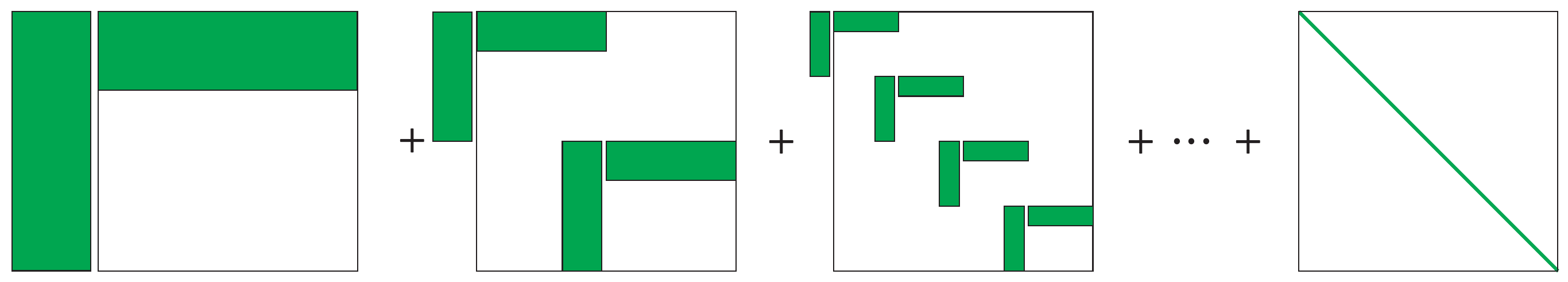}
    \end{center}
    \caption{(Contiguous) PSD MLR matrix given as a 
    sum of block diagonal matrices with each block being low rank.
    The coefficients of the factors are depicted in green.}
    \label{fig-mlr-blockdiag-form}
\end{figure}

An MLR matrix~\cite{parshakova2023factor} is a row and column 
permutation of a sum of matrices, each one a block diagonal refinement 
of the previous one, 
with all blocks low rank, given in the factored form.
We focus on the special case of symmetric positive 
semidefinite (PSD) MLR matrices.

An $n \times n$ contiguous PSD MLR matrix $\Sigma$ with $L$ levels has the form
\BEQ\label{e-mlr-a}
\Sigma = \Sigma^1 + \cdots + \Sigma^L,
\EEQ
where $\Sigma_l$ is a PSD block diagonal matrix,
\[
\Sigma_l = \blkdiag(\Sigma_{l,1}, \ldots, \Sigma_{l,p_l} ), \quad l=1, \ldots , L,
\]
where $\blkdiag$ is the direct sum of blocks 
$\Sigma_{l,k} \in \reals^{n_{l,k} \times n_{l,k}}$ for $k=1, \ldots, p_l$.
Here $p_l$ is the size of the partition at level $l$, 
and 
\[
\sum_{k=1}^{p_l} n_{l,k} = n, \quad
l=1,\ldots, L.
\]
Throughout this paper we consider $L \geq 2$ and $p_L=n$,
therefore $\Sigma_L$ is a diagonal matrix. 
Also for all $l=1, \ldots, L$ define matrices
\[
\Sigma_{l+} = \Sigma_l + \cdots + \Sigma_L,
\qquad
\Sigma_{l-} = \Sigma_1 + \cdots + \Sigma_l.
\]
By definition, we have $\Sigma=\Sigma_{1+}=\Sigma_{L-}$.

The block dimensions on level $l$ partition the $n$ indices
into $p_l$ groups, which are contiguous. 
Let $J_1, \ldots, J_L$ be partitions of the set
$\{1, \ldots, n\}$. (By symmetry of $\Sigma_l$, these partitions are 
the same for rows and columns.)

For each $l=1, \ldots, L$, 
the level $l$ partition of the indices is the set of $p_l$ 
index sets
\[
J_l = \left \{\{1, \ldots, n_{l,1}\}, ~
\{n_{l,1}+1, \ldots, n_{l,1}+ n_{l,2}\}, ~\ldots,~
\{n-n_{l,p_l}+1, \ldots, n\} \right \}.
\]
We require that these partitions be hierarchical, meaning that
for all $l=2, \ldots, L$,
the partition $J_l$ is a \emph{refinement} of $J_{l-1}$.
We write
\[
J_l \preceq J_{l-1}
\]
to indicate that for every index set
$X \in J_l$, there exists index set $Y \in J_{l-1}$
such that $X \subseteq Y$.

We require that blocks on level $l$ have rank not exceeding $r_l$,
given in the factored form as 
\[
\Sigma_{l,k} = F_{l,k} F_{l,k}^T, \quad F_{l,k}\in \reals^{n_{l,k}\times r_l}, \quad
l=1, \ldots, L-1, \quad k=1, \ldots, p_l,
\]
and refer to $F_{l,k}$ as the factor (of block $k$ on level $l$).

Define a diagonal matrix 
$D = \Sigma_L$, which forces $r_L=1$.
See figure~\ref{fig-mlr-blockdiag-form}.
We refer to $r= r_1+ \cdots + r_{L-1} + 1$ as the MLR-rank of $\Sigma$.
The MLR-rank of $A$ is in general not the same as the rank of $\Sigma$.
We refer to $(r_1, \ldots,  r_{L-1}, 1)$ as the rank allocation.

% \subsubsection{Two-matrix form}
\paragraph{Factor form.}
\begin{figure}
    \begin{center}    
    \includegraphics[width=0.53\textwidth]{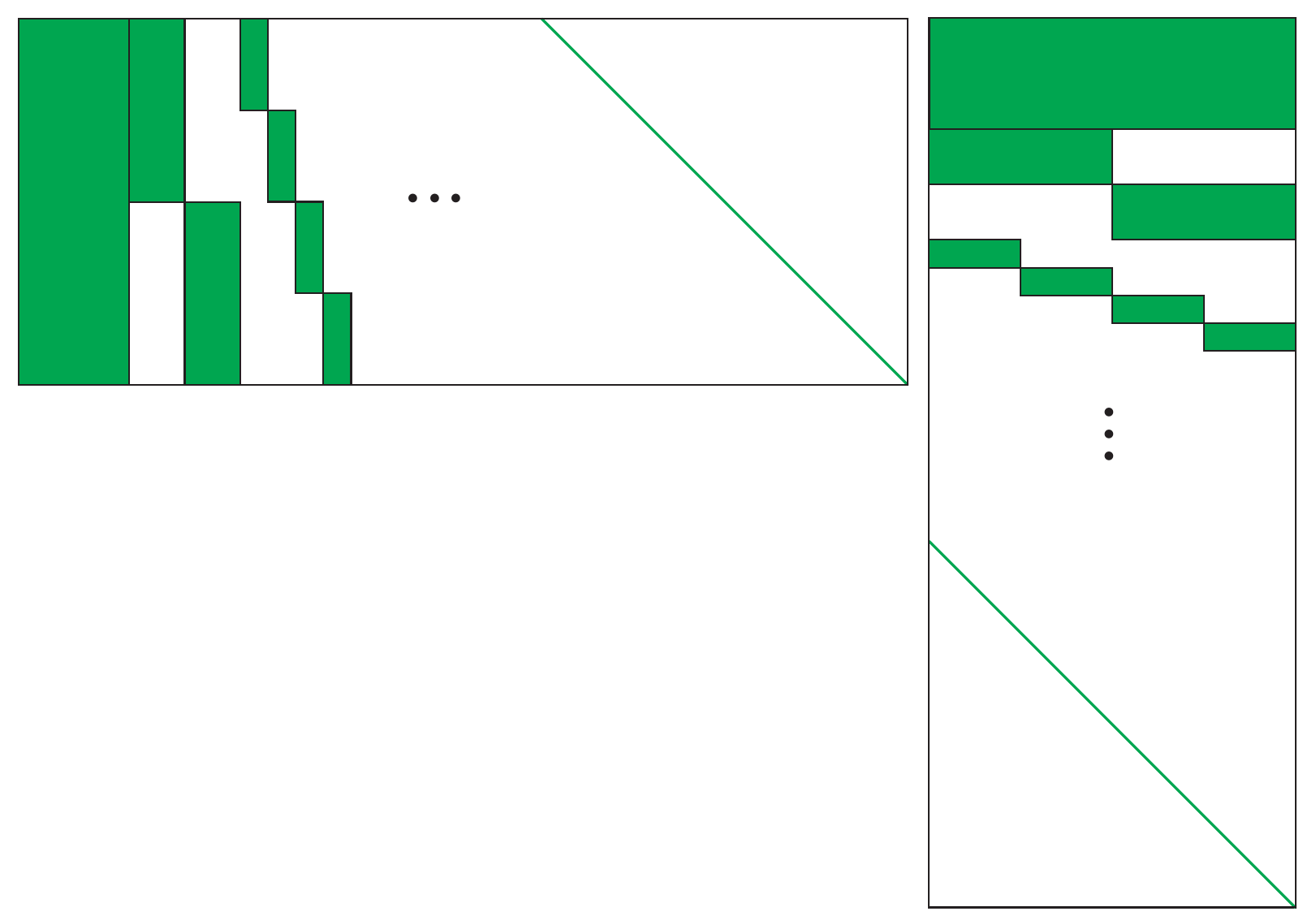}
    \end{center}
    \caption{(Contiguous) PSD MLR matrix given as a product of two sparse 
    structured matrices.
    The coefficients of the factors are depicted in green.}
    \label{fig-mfm-factor-form}
\end{figure}
For each level $l=1, \ldots, L-1$ define 
\[
 F_l = \blkdiag(F_{l,1}, \ldots, F_{l,p_l}) \in \reals^{n \times p_lr_l}.
\]
Then we have 
\[
\Sigma_l =  F_l F_l^T, \quad l=1, \ldots, L-1.
\]
Define 
\[
 F = \left[ \begin{array}{ccc} 
 F_1 & \cdots &  F_{L-1}
\end{array}\right] \in \reals^{n \times s},
\]
with $s = \sum_{l=1}^{L-1} p_l r_l$.
Then we can write $\Sigma$ as
\[
\Sigma =  \left[ \begin{array}{ccc} 
F & D^{1/2}
\end{array}\right] \left[ \begin{array}{ccc} 
F & D^{1/2}
\end{array}\right]^T
= FF^T + D,
\]
where $F$ has $s$ columns, and a very specific 
sparsity structure, with column blocks that are block diagonal, 
and $D$ is diagonal,
see figure~\ref{fig-mfm-factor-form}.

Define $F_{l+}$ as the concatenation of left factors from levels 
$l, \ldots, L-1$,
and similarly $F_{l-}$, \ie,
\[
 F_{l+} = \left[ \begin{array}{ccc} 
 F_l & \cdots &  F_{L-1}
\end{array}\right], \qquad
 F_{l-} = \left[ \begin{array}{ccc} 
 F_1 & \cdots &  F_l
\end{array}\right].
\]
Thus the number of nonzero coefficients in $F_{l+}$ is $n\sum_{l'=l}^{L-1}r_{l'}$
and in $F_{l-}$ is $n\sum_{l'=1}^{l}r_{l'}$.
By definition, we also have $F = F_{1+} = F_{(L-1)-}$.

\paragraph{Compressed factor form.}

\begin{figure}
    \begin{center}    
    \includegraphics[width=0.15\textwidth]{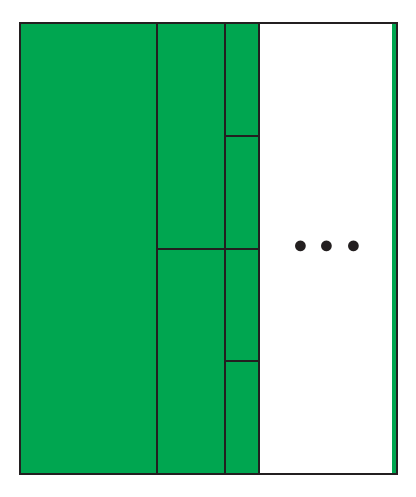}
    \end{center}
    \caption{(Contiguous) PSD MLR matrix given in
    compressed form.
    % The coefficients of the factor are depicted in green.
    }
    \label{fig-mlr-compressed-form}
\end{figure}
We can also arrange the factors into one dense matrix with 
dimensions $n \times r$.
We vertically stack the factors at each level to form matrices
\[
\bar{F}^l = \left[ \begin{array}{c} 
F_{l,1} \\ \vdots \\ F_{l,p_l}
\end{array} \right] \in \reals^{n \times r_l}, \quad l=1, \ldots, L-1,
\]
and lastly a diagonal of matrix $D$, $\diag(D)\in \reals^n$. 
We horizontally stack these matrices to obtain one matrix
\[
\bar{F} = \left[ \begin{array}{ccc}
\bar{F}^1 & \cdots & \bar{F}^{L-1}
\end{array} \right] \in \reals^{n \times (r-1)}.
\]
All of the coefficients in the factors of a contiguous MLR matrix are contained 
in this matrix and vector $\diag(D)$,
see figure~\ref{fig-mlr-compressed-form}.  To fully specify a contiguous 
MLR matrix, we need to give
the block dimension $n_{l,k}$ for $l=1, \ldots, L$, $k=1, \ldots, p_l$,
and the ranks $r_1, \ldots, r_L$.

\paragraph{PSD MLR matrix.}
We reviewed the contiguous PSD MLR matrix.
PSD MLR matrix is given by the symmetric permutation of rows and columns
of a contiguous PSD MLR matrix.
Therefore, the PSD MLR matrix uses a general hierarchical 
partition of the index set.

\paragraph{Example.}
To illustrate our notation we give an example with $L=4$ levels, 
$p_1=1$, 
with the second level partitioned into $p_2=2$ groups, and the third level
partitioned into $p_3=4$ groups.
We take $n=5$, with block row (and column) dimensions 
\[
\begin{array}{l}
n_{1,1}=5\\
n_{2,1}=3, \quad
n_{2,2}=2,  \\
n_{3,1}=1, \quad
n_{3,2}=2, \quad
n_{3,3}=1, \quad
n_{3,4}=1 \\ 
n_{4,1}=1, \quad
n_{4,2}=1, \quad
n_{4,3}=1, \quad
n_{4,4}=1, \quad 
n_{4,5}=1.
\end{array}
\]
The sparsity patterns of $\Sigma_1$, $\Sigma_2$ and $\Sigma_3$ are shown below,
with $*$ denoting a possibly nonzero entry, and all other entries zero.
(The sparsity pattern of $\Sigma_4$ matches that of a diagonal matrix.)
\[
\Sigma_1 = \left[ \begin{array}{ccccccc}
*&*&*&*&*\\
*&*&*&*&*\\
*&*&*&*&*\\
*&*&*&*&*\\
*&*&*&*&*
\end{array}\right], \qquad
\Sigma_2 = \left[ \begin{array}{cccccc}
*&*&* \\
*&*&*  \\
*&*&* \\
&&&*&* \\
&&&*&* 
\end{array}\right], \qquad
\Sigma_3 = \left[ \begin{array}{cccccc}
*&& \\
&*&*  \\
&*&* \\
&&&*& \\
&&&&* 
\end{array}\right].
\]
If we have ranks $r_1=2$, $r_2=1$, $r_3=1$, and $r_4=1$, the MLR-rank is $r=5$,
with factor sparsity pattern as below,
\[
F_1 = \left[ \begin{array}{cc}
*&*\\
*&*\\
*&*\\
*&*\\
*&*
\end{array}\right], \qquad
F_2 = \left[ \begin{array}{cc}
*& \\
*& \\
*& \\
&* \\
&* 
\end{array}\right], \qquad
F_3 = \left[ \begin{array}{ccccc}
*&& \\
&*  \\
&* \\
&&*& \\
&&&* 
\end{array}\right].
\]
This means that $\Sigma_1$ has rank $2$, the $p_2=2$ blocks in $\Sigma_2$ each have rank $1$,
and the $p_3=4$ blocks in $\Sigma_3$ also have rank $1$.

\subsection{Partition notation}
In this paper we consider matrices that are block diagonal, \eg, $F_l$, and
matrices formed by concatenation of 
block diagonal matrices, \eg, $F_{l+}$. 
To formally describe the row and column sparsity patterns of these matrices,
we define the following operators.

Define an operator $\tilde \cJ$,
that
for any block diagonal matrix $B \in \reals^{m \times n}$  
returns its column index partition.
Similarly, define operator $\tilde \cI$  
to return the row index partition of $B$.
Note by definition $\tilde \cI(B) = \tilde \cJ(B^T)$.

Define operators $\cI$ and $\cJ$ that for any (horizontal or vertical)
concatenation of block diagonal matrices
$B=\left[ \begin{array}{ccc} 
B_1 & \cdots &  B_c
\end{array}\right] \in \reals^{m \times n}$ 
return lists of partitions for each block diagonal matrix
\[
\cJ(B) =  
 (\tilde \cJ(B_1), \ldots,  \tilde \cJ(B_c)),
\qquad
\cI(B) = 
 (\tilde \cI(B_1), \ldots,  \tilde \cI(B_c)),
\]
% a list of (row or column) index partitions of every block diagonal in $B$.

We say a partition \emph{refines} a list of partitions if it refines 
each partition in that list. 
Conversely, we say a list of partitions refines a partition if every partition in the
list refines that partition.
We denote this relation by $\preceq$.

Finally, define the sparsity pattern of any  
$B \in \reals^{m \times n}$ as
\[
\supp(B) = \{ (i,j) \mid B_{ij} \neq 0, ~i = 1, \ldots, m, ~j=1, \ldots, n \}.
\]

\begin{remark}
    If $B, C \in \reals^{m\times n}$ are concatenations of block diagonal matrices 
    with
    $\supp(B) = \supp(C)$, then $\cI(B) = \cI(C)$ and 
    $\cJ(B) = \cJ(C)$.
\end{remark}

\paragraph{Example.} Applying these operators to the matrices from the previous 
section, we get 
\[
\cI(\Sigma_l) = \cJ(\Sigma_l) = \cI(F_l) = J_l,
\]
and 
\BEAS 
\cI(F_{l-}) &=& (J_1, \ldots, J_l) \\
\cJ(F_l) &=& \left \{ \{1, \ldots, r_l\}, ~
\{r_l+1, \ldots, 2r_l\}, ~
\ldots,~
\{(p_l-1)r_l+1, \ldots, p_lr_l\} \right \} \\
\cI(F_{l+}) &=& (J_l, \ldots, J_{L-1}).
\EEAS 
We also have
\[
\cI(F_{l+}) \preceq \cI(F_l) \preceq \cI(F_{l-}),
\]
and 
\[
\supp(\Sigma_l) = \supp(\Sigma_{l+}).
\]

\subsection{Problem setting}
We consider a multilevel factor model,
\BEQ\label{e-hier-factor-model}
y = Fz + e,
\EEQ
where $F \in \reals^{n \times s}$ is structured factor loading matrix,
$z \in \reals^s$ are factor scores, with $z \sim \mathcal N(0, I_s)$, and 
$e \in \reals^n$ are the 
idiosyncratic terms, with $e \sim \mathcal N(0, D)$.

We assume that the $n$ features can be hierarchically partitioned,
with specific factors explaining the correlations within each group of 
this hierarchical partition.
This can be modeled by taking $F$ to be the factor matrix of PSD MLR. 
Then $y \in \reals^n$ is a Gaussian random vector with zero mean 
and covariance matrix $\Sigma$ that is PSD MLR,  
\[
\Sigma 
% = \left[ \begin{array}{ccc} 
% F & D^{1/2}
% \end{array}\right] \left[ \begin{array}{ccc} 
% F & D^{1/2}
% \end{array}\right]^T
= FF^T + D.
\]

We assume we have access to hierarchical partition and rank allocation.
Therefore, we reorder $n$ features so that the groups in hierarchical partition
correspond to contiguous index ranges.
We seek to fit the coefficients of $F\in \reals^{n \times s}$ and diagonal 
$D \in \reals^{n \times n}$ (with $\diag(D)>0$) from the observed samples. 

We assume $s \ll n$, \ie, number of factors is smaller than
the number of features.

\section{Fitting methods}

In this paper, we estimate parameters $F$ and $D$ using the maximum 
likelihood estimation (MLE).
This approach is different from that in~\cite{parshakova2023factor}, 
which focuses on fitting the PSD MLR matrix to the empirical covariance matrix
using a Frobenius norm-based loss.
Notably, the Frobenius norm is not an appropriate loss for fitting 
covariance models. First, the Frobenius norm is 
coordinate-independent, it treats all coordinates equally,
whereas MLE accounts for coordinate-specific differences, 
where changes across different coordinates have varying implications.
This can lead to covariance models with small eigenvalues when using 
the Frobenius norm, a situation that MLE inherently guards against.
Second, the Frobenius norm-based loss is distribution-agnostic.
In contrast, MLE takes advantage of the known distribution of the data.
Nevertheless, there is an intrinsic connection between the MLE and Frobenius
norm, which we detail in \S\ref{a-second-approx-ll} of the appendix.

\subsection{Frobenius norm-based estimation}
One way to estimate coefficients of matrices $F$ and $D$ is by 
minimizing Frobenius norm-based distance with sample covariance.
This means solving the following optimization problem
\BEQ\label{e-frob-fitting-prob}
\begin{array}{ll}
\mbox{minimize} & \|FF^T+D - \hat \Sigma\|_F^2 \\
\mbox{subject to} & FF^T+D~\mbox{is PSD MLR},
\end{array}
\EEQ
with the hierarchical partition and sparsity structure of $F$
(number of levels, block dimensions, and ranks) predefined and fixed,
as previously proposed in \cite{parshakova2023factor}.

Since the problem \eqref{e-frob-fitting-prob} is nonconvex,
\cite[\S4]{parshakova2023factor} introduce two complementary block coordinate descent 
methods to find an approximate solution.
For example, alternating least squares minimizes the fitting error over the
left factors, then over the right factors, and so on.
The second method updates factors at one level in each iteration
by minimizing the fitting error while cycling over the levels.

\subsection{Maximum likelihood estimation}
Alternatively we can estimate matrices $F$ and $D$ using MLE.
Suppose we observe samples $y_1, \ldots, y_N \in \reals^n$, 
organized in the matrix form as
\[
Y = \left[ \begin{array}{c} 
y_1^T \\ \vdots \\ y_N^T
\end{array} \right]  \in \reals^{N \times n}.
\]
The log-likelihood of $N$ samples is
\BEQ\label{e-ll-observations}
\ell(F, D; Y) = 
- \frac{nN}{2} \log(2\pi) - \frac{N}{2}\log \det(FF^T+D) - 
\frac{1}{2}\Tr((FF^T+D)^{-1} Y^T Y).
\EEQ
For structured $F$, directly maximizing the log-likelihood $\ell(F, D; Y)$ is difficult. 
Instead, the expectation-maximization (EM) algorithm~\cite{dempster1977maximum}
is the preferred approach for MLE.

\paragraph{Simplification via data augmentation.}\label{sec-data-augmentation}
Difficult maximum likelihood problems can be simplified by data augmentation.
Suppose along with $Y$ we also observed latent data $z_1, \ldots, z_N \in \reals^{s}$, 
organized in matrix $Z\in \reals^{N \times s}$.
Then the log-likelihood of complete data $(Y,Z)$ for 
model~\eqref{e-hier-factor-model} is
\BEQ\label{e-data-augmentation}
\ell(F, D; Y, Z) 
= -  \frac{(n+s)N}{2}\log(2\pi) -\frac{N}{2}\log \det D - 
\frac{1}{2} \| (Y- Z F^T)  D^{-1/2}\|_F^2 - \frac{1}{2} \|Z\|_F^2.
\EEQ
Maximizing the $\ell(F, D; Y, Z)$ with respect to $F$ and $D$ is now tractable. 
First, since $D$ is diagonal, when $F$ is known solving for $D$ is trivial.
Second, note that $\ell(F, D; Y, Z)$ is separable across the rows of $F$.
The nonzero coefficients in each row of $F$ can be found by 
solving the least squares problem.

For example, consider a simple factor model, 
where $F$ is just a dense low rank matrix.
Then from the optimality conditions, the solution to \eqref{e-data-augmentation} 
is given by
\BEQ\label{e-data-augmentation-sol-FFtD}
 F = Y^T Z (Z^T Z)^{-1}, \qquad  D = \frac{1}{N} 
\diag(\diag((Y- Z  F^T)^T(Y- Z F^T))).
\EEQ

Since we only observe $Y$ while $Z$ is missing, we use the EM algorithm
to simplify the problem through data augmentation.

\section{EM algorithm}
EM algorithm iterates expectation and maximization steps until convergence.
After each pair of E and M steps it can be shown that the log-likelihood of 
the observed data is non-decreasing, with equality at a local optimum.

\subsection{Expectation step}\label{sec-e-step}
% Let $F^0$ and $D^0$ be initial guesses for the parameters $F$ and $D$.
In the expectation step we compute the conditional expectation of complete data
log-likelihood 
with respect to the conditional distribution $(Y, Z \mid Y)$ governed by the
 the current estimate of parameters $F^0$ and $D^0$:
\BEQ
Q(F, D; F^0, D^0) = \Expect \left ( \ell(F, D; Y, Z)  \mid Y, F^0, D^0  \right ). 
\label{e-q-e-step}
\EEQ
To evaluate the $Q(F, D; F^0, D^0)$, we need to compute several expectations.
First, using \eqref{e-hier-factor-model} we have
\BEAS
\Cov(y,z) &=& \Expect F zz^T = F \\
\Cov(y, y) &=& FF^T + D = \Sigma.
\EEAS 
Thus $(z, y)$ is a Gaussian random vector with zero mean and covariance
\[
\Cov\left( (z, y), (z, y) \right) = \left[ \begin{array}{cc} 
I_{s} & F^T \\ 
F & \Sigma
\end{array} \right].
\]
Second, the conditional distribution $(z_i \mid y_i, F^0, D^0)$ is Gaussian, 
\[
 \mathcal N\left( {F^0}^T(\Sigma^0)^{-1}y_i, I_{s} - {F^0}^T(\Sigma^0)^{-1}F^0 \right).
\]

Using the omitted derivations in \S\ref{apx-e-step},
we can show that \eqref{e-q-e-step} equals
\BEA\label{e-expect-step}
Q(F, D; F^0, D^0) &=& -  \frac{(n+s)N}{2}\log(2\pi) -\frac{N}{2}\log \det D - 
\frac{1}{2}\Tr (W) \nonumber\\
&& - \frac{1}{2}\Tr \left(D^{-1}
(Y^TY - 2F V + F W F^T) \right),
\EEA
where we defined matrices $V\in \reals^{s \times n}$ and 
$W\in \reals^{s \times s}$ as
\BEA
V &=& \sum_{i=1}^N\Expect \left ( z_i \mid y_i, F^0, D^0 \right )y_i^T = 
{F^0}^T (\Sigma^0)^{-1}Y^TY \label{e-matrix-V}\\
W &=& \sum_{i=1}^N\Expect \left ( z_i z_i^T \mid y_i, F^0, D^0 \right ) \nonumber\\
&=&N(I_{s} - 
{F^0}^T(\Sigma^0)^{-1}F^0) 
+ {F^0}^T(\Sigma^0)^{-1}Y^TY (\Sigma^0)^{-1} F^0.\label{e-matrix-W}
\EEA 
\begin{remark}
Note that $(I_{s} - 
{F^0}^T(\Sigma^0)^{-1}F^0) \succ 0$, as it is a Schur complement of matrix
\[ \left[ \begin{array}{cc} 
I_s & F^T \\ 
F & \Sigma
\end{array} \right] \succ 0.
\] 
Consequently, it follows that $W \succ 0$.
\end{remark}

\subsection{Maximization step}\label{sec-m-step}
In the maximization step we find updated parameters $F^1$ and $D^1$ 
by solving the following problem
\BEQ\label{e-max-step}
\begin{array}{ll}
\mbox{maximize} & Q(F, D; F^0, D^0) \\
\mbox{subject to} & \left[ \begin{array}{cc} 
F & D^{1/2}
\end{array}\right]~\mbox{is the factor of PSD MLR}.
\end{array}
\EEQ

Similar to~\eqref{e-data-augmentation}, the maximization problem~\eqref{e-max-step}
is tractable.
Observe, $Q(F, D; F^0, D^0)$ is separable across the rows of $F$ (and respective
diagonal elements of $D$).
Moreover, using optimality conditions,
the nonzero coefficients in each row of $F$ can be determined by 
solving the least squares problem.
For efficiency, we can group the rows by their sparsity pattern and instead solve
the least squares problems for each row sparsity pattern of $F$ at once,
forming resulting matrix $F^1$,
see \S\ref{sec-em-computation}.
Having $F^1$, the diagonal matrix is then equal to
\[
 D^1 = \frac{1}{N} \diag(\diag(Y^TY - 2F^1 V + F^1 W (F^1)^T)).
\]
Thus $F^1$ and $D^1$ are the optimal solutions to problem~\eqref{e-max-step},
which we can also compute efficiently as discussed in \S\ref{sec-efficient-comp}.

\subsection{Initialization}
EM algorithm is a maximization-maximization procedure \cite[\S8.5]{hastie2009elements}, 
therefore, it converges to at least a local maximum.
The trajectory of the EM algorithm depends on the initial values of $F^0$ and $D^0$.
We have observed
that, depending on the initialization, it can converge to different local maxima. 
Additionally, when a good initial guess is not available,
we have also observed that initializing matrices using
a single sweep of the block coordinate descent method \cite[\S4.2]{parshakova2023factor} 
from the top to bottom level works well.

\section{Efficient computation}\label{sec-efficient-comp}

\subsection{Inverse of PSD MLR}\label{s-inverse}
In the maximization step, evaluating matrices $V$~\eqref{e-matrix-V} and 
$W$~\eqref{e-matrix-W}
requires solving linear systems with the PSD MLR matrix. 
We will first address the efficient computation of $\Sigma^{-1}$, \ie,
\[
(F_1F_1^T + \cdots + F_{L-1}F_{L-1}^T + D)^{-1}.
\]
We will show that the inverse of the PSD MLR matrix is the MLR matrix with 
the same hierarchical partition and rank allocation, and
\[
\Sigma^{-1} = -H_1 H_1^T - \cdots - H_{L-1}H_{L-1}^T + D^{-1},
\]
where $H_l \in \reals^{n \times p_lr_l}$ is a factor at level $l$
with the same sparsity structure as $F_l$.

% Thus the number of nonzero coefficients in $F_{l+}$ is $n\sum_{l'=l}^{L-1}r_{l'}$
% and in $F_{l-}$ is $n\sum_{l'=1}^{l}r_{l'}$.
We compute the coefficients of the inverse by recursively applying the
Sherman-Morrison-Woodbury (SMW) matrix identity.

\subsubsection{Properties of structured matrices}\label{sec-properties-sparse-structured}
We begin by giving useful properties of our structured matrices.
Consider a factor matrix on level $l$, $F_l \in \reals^{n \times p_lr_l}$, 
with $p_l$ diagonal blocks of size
$n_{l,k} \times r_l$,
and row index partition set $J_l$, for all $k=1, \ldots, p_l$.

\begin{remark}
Lemma \ref{lem-blockdiag-prod} states that if block diagonal matrices 
$B$ and $C$ are such that
$\cJ(B) \preceq \cI(C)$, then $BC$ is block diagonal with
$\cJ(BC) = \cJ(C)$, \eg,
see below. Moreover, if $\cI(B) = \cJ(B)$, then $\supp(BC) = \supp(C)$.
\begin{center}    
\includegraphics[width=0.52\textwidth]{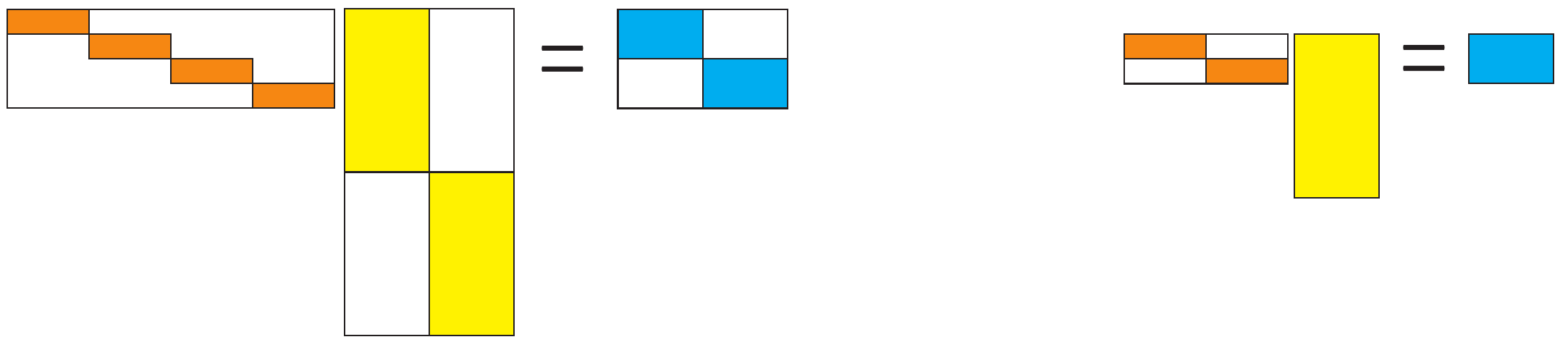}
\end{center}
\end{remark}

\begin{remark}\label{remark-properties-structured}
The following properties are based on Lemma~\ref{lem-blockdiag-prod},
and they will be useful in the next section.

\begin{enumerate}
\item 
Matrix $F_l F_l^T \in \reals^{n \times n}$ is a block diagonal matrix
with blocks of size $n_{l,k} \times n_{l,k}$, 
with $\cI(F_l F_l^T) =\cJ(F_l F_l^T)= J_l$,
\eg, see illustration below.
\begin{center}    
\includegraphics[width=0.35\textwidth]{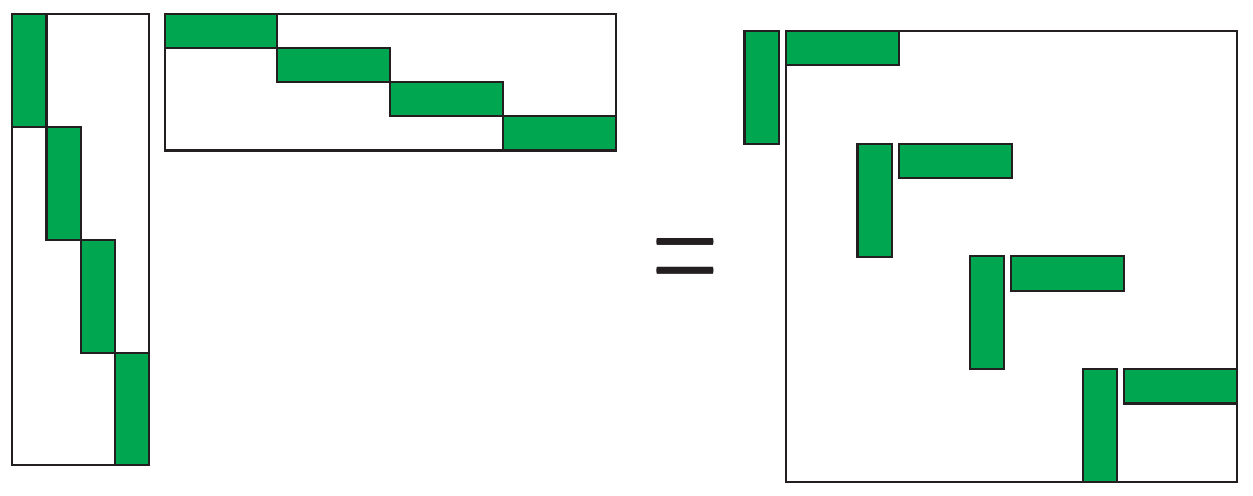}
\end{center}
For all $l' \geq l$, $J_{l'} \preceq J_l$ 
implies $\supp(F_{l'} F_{l'}^T) \subseteq \supp(F_l F_l^T)$.
Then for matrix
\[
F_{(l+1)+} F_{(l+1)+}^T = \sum_{l'=l+1}^{L-1} F_{l'} F_{l'}^T,
\]
we obtain $\supp(F_{l+1} F_{l+1}^T) =\supp(F_{(l+1)+} F_{(l+1)+}^T)$.

\item 
For matrix $\Sigma_{(l+1)+}$
it holds $\supp(\Sigma_{(l+1)+}) = \supp(F_{(l+1)} F_{(l+1)}^T)$.

\item 
The inverse of a block diagonal matrix is a block diagonal matrix 
consisting of the inverses of each block. 
Thus for 
\[
\Sigma_{(l+1)+}^{-1}=(F_{(l+1)+} F_{(l+1)+}^T + D)^{-1}
\]
we get $\supp(\Sigma_{(l+1)+}^{-1})=\supp(\Sigma_{(l+1)+})$.

\item 
Since $\cI(\Sigma_{(l+1)+}^{-1}) = \cJ(\Sigma_{(l+1)+}^{-1}) \preceq \cI(F_l)$,
for 
$M_0 =\Sigma_{(l+1)+}^{-1}F_l$, $\supp(M_0) = \supp(F_l)$.

\begin{center}    
\includegraphics[width=0.3\textwidth]{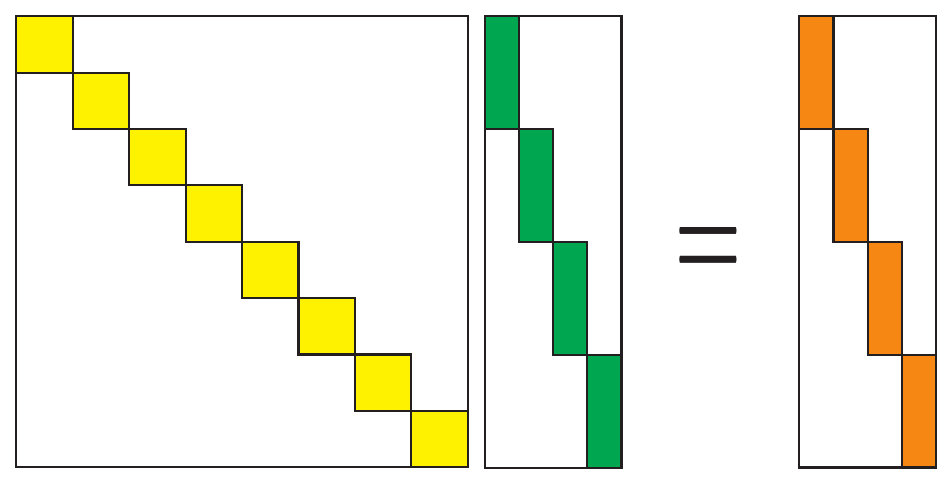}
\end{center}
Thus matrix-vector product with $M_0$ can be computed in the order of 
$\sum_{k=1}^{p_l} n_{l,k} r_l= nr_l$ operations.

\item 
Since $\cJ(F_l^T) \preceq \cI(F_{(l-1)-})$,
we have $\cJ(M_0^T F_{(l-1)-}) = \cJ(F_{(l-1)-})$.
Further, since $\cJ(\Sigma_{(l+1)+})\preceq \cI(F_{(l-1)-}) $, it follows
\[
\supp(\Sigma_{(l+1)+}^{-1} F_{(l-1)-}) = \supp(F_{(l-1)-}).
\]

\item For $F_l^T M_0 \in \reals^{p_lr_l \times p_l r_l}$ 
% is a block diagonal 
% matrix with $p_l$ blocks of size $r_l \times r_l$, 
it holds $\cI(F_l^T M_0) = \cJ(F_l^T M_0) = \cJ(F_l)$,
see figure below.

\begin{center}    
\includegraphics[width=0.47\textwidth]{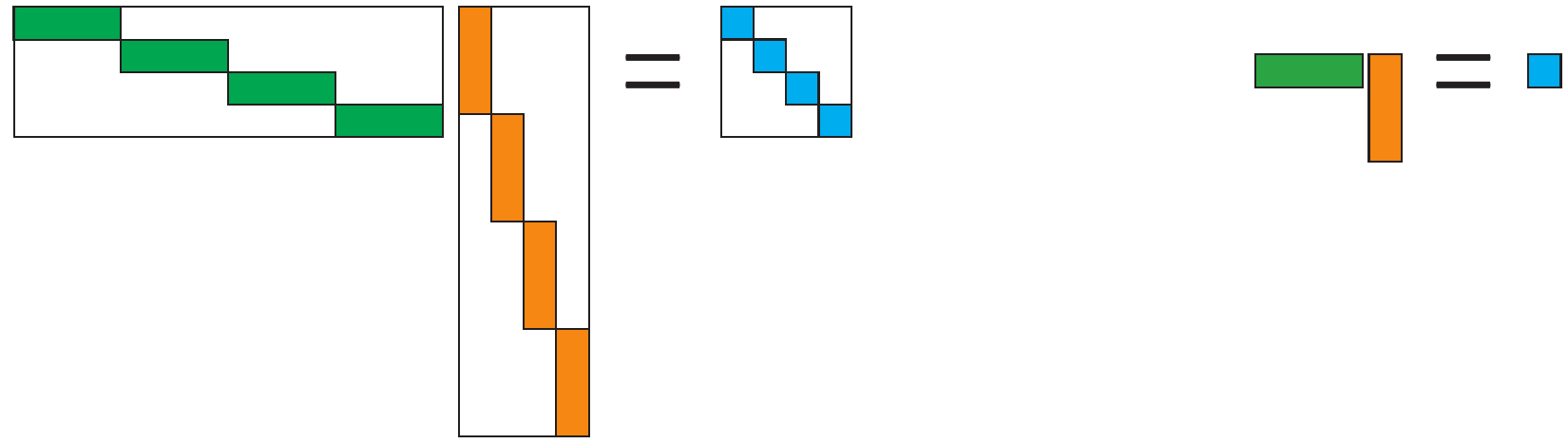}
\end{center}
It is straightforward to check that each of the blocks is PSD.
\end{enumerate}
\end{remark}

\subsubsection{Computing the inverse}\label{sec-rec-smw}
We show that $\Sigma^{-1}$ is an MLR matrix with factors having the same sparsity 
pattern as $\Sigma$.
To establish this, we employ SMW matrix identity
\[
(FF^T+D)^{-1}=D^{-1} - D^{-1}F(I_s+F^TD^{-1}F)^{-1}F^TD^{-1}.
\]
We derive 
\BEA
% (F_{l+} F_{l+}^T + D)^{-1} 
% &=&(F_{(l+1)+} F_{(l+1)+}^T + D)^{-1} - H_l H_l^T,
\Sigma_{l+}^{-1} &=& \Sigma_{(l+1)+}^{-1} - H_l H_l^T,
\label{e-inverse-woodbury}
\EEA
% or simply $\Sigma_{l+}^{-1} = \Sigma_{(l+1)+}^{-1} - H_l H_l^T$,
where we defined matrix
\[
H_l = \Sigma_{(l+1)+}^{-1}F_l(I_{p_lr_l} + F_l^T\Sigma_{(l+1)+}^{-1}F_l)^{-1/2},
\]
see \S\ref{apx-inv-comp} for details.
Remark~\ref{remark-properties-structured}
implies that $\supp(H_l) = \supp(\Sigma_{(l+1)+}^{-1}F_l) = \supp(F_l)$. 
% Further we have
% \[
% (F_{l+} F_{l+}^T + D)^{-1} = (F_{(l+1)+} F_{(l+1)+}^T + D)^{-1} - H_l H_l^T.
% \]
Applying recursion \eqref{e-inverse-woodbury} from the bottom to the top level we get
\[
\Sigma^{-1} = -H_1 H_1^T - \cdots - H_{L-1}H_{L-1}^T + D^{-1}.
\]
Combining, we establish that $\Sigma^{-1}$ is an MLR matrix
with the same hierarchical partition as~$\Sigma$.

\paragraph{Recursive SMW algorithm.}
We now show that the complexity of computing the coefficients of the MLR matrix
$\Sigma^{-1}$ is
$O(nr^2 + p_{L-1} r_{\max}r^2)$ and extra memory used is less than 
$3nr + 2 p_{L-1}r_{\max}r$,
where $r_{\max} = \max\{r_1, \ldots, r_L\}$.
To do so, we recursively compute the coefficients of the matrices
\BEQ\label{e-inv-recursive-terms}
\Sigma_{l+}^{-1} F_{(l-1)-}, \qquad H_l,
\EEQ
from the bottom to the top level. 

Suppose we have $n\sum_{l'=1}^{l}r_{l'}$ coefficients of 
$\Sigma_{(l+1)+}^{-1} F_{l-}$. 
This implies that we have the coefficients of 
$M_0=\Sigma_{(l+1)+}^{-1}F_l$.
We now show how to compute \eqref{e-inv-recursive-terms} 
using SMW matrix identity~\eqref{e-inverse-woodbury}.
\begin{enumerate}
\item Compute $M_1=M_0^T F_{(l-1)-}$ in 
$O(nr_l\sum_{l'=1}^{l-1}r_{l'})$
and store its $p_lr_l\sum_{l'=1}^{l-1}r_{l'}$ coefficients,
since for $l'\leq l-1$ computing $M_0^TF_{l'}$ takes
$nr_lr_{l'}$ operations, and compact form of 
$F_{(l-1)-}$ 
has $\sum_{l'=1}^{l-1}r_{l'}$ columns.
\item Compute $M_2=(I_{p_lr_l}+ 
F_l^TM_0)^{-1}$
in $O(nr_l^2 + p_lr_l^3)$ and store its $p_lr_l^2$ coefficients. 
Compute $H_l=M_0(I_{p_lr_l}+ 
F_l^TM_0)^{-1/2}$ in $O(nr_l^2 + p_lr_l^3)$ and store its $nr_l$ coefficients.
Note that computing $I_{p_lr_l} + F_l^TM_0$
requires $O(nr_l^2)$ operations, 
and its eigendecomposition, $I_{p_lr_l} + F_l^TM_0=Q_l \Lambda_l Q_l^T$, to 
compute $H_l$ takes $O(p_lr_l^3)$ operations.
\item Compute $M_3=M_2M_1$
in $O(p_lr_l^2\sum_{l'=1}^{l-1}r_{l'})$ and store its 
$p_l r_l\sum_{l'=1}^{l-1}r_{l'}$ coefficients,
since 
$\cI(M_2)=\cJ(M_2) \preceq \cI(M_1)$
% the sparsity of $M_2$ refines the row sparsity of block diagonals in $M_1$, 
and compact form of $M_1$ 
has $\sum_{l'=1}^{l-1}r_{l'}$ columns.
Note that $\supp(M_3)=\supp(M_1)$. 
\item Compute $M_4 = M_0 M_3$ in 
$O(nr_l\sum_{l'=1}^{l-1}r_{l'})$
and store its $n\sum_{l'=1}^{l-1}r_{l'}$ coefficients,
since $\cJ(M_0) \preceq \cI(M_3)$,
% the column sparsity of $M_0$ refines the row sparsity of block 
% diagonals in $M_3$, 
and compact form of $M_3$ 
has $\sum_{l'=1}^{l-1}r_{l'}$ columns.
Note that $\supp(M_4) = \supp(F_{(l-1)-})$. 
\item Compute $M_5 = \Sigma_{(l+1)+}^{-1} F_{(l-1)-} - M_4$ in 
$n\sum_{l'=1}^{l-1}r_{l'}$ and store its $n\sum_{l'=1}^{l-1}r_{l'}$ coefficients.
\end{enumerate}
Therefore, the complexity at the level $l$ is
\[
O\left ( (nr_l + p_l r_l^2) \sum_{l'=1}^l r_{l'}  \right).
\]

Finally, we conclude that the total complexity  is
\[
T(n) = \sum_{l=1}^{L-1} O \left ( (nr_l + p_l r_l^2) \sum_{l'=1}^l r_{l'} \right ) 
= O(nr^2 + p_{L-1} r_{\max}r^2),
\]
and extra storage used is less than $3nr + 2 p_{L-1}r_{\max}r$.

Recall that $s=\sum_{l=1}^{L-1}p_l r_l \ll n$, therefore, we have $p_{L-1}\ll n$.
This implies that the time complexity is linear in $n$.

If we assume that the rank allocation is uniform 
$r_1=\cdots=r_{L-1}=\tilde r$
and that each block on one level is split into two nearly equal-sized blocks 
on the next level, $p_l=2^{l-1}$,
then the total complexity and storage are respectively
\[
T(n) = O(n \tilde r^2L^2 + 2^L \tilde r^3 L), \qquad
3n\tilde rL + 2^L \tilde r^2 L.
\]
Using the assumption that $s \ll n$ and $s = (2^{L-1}-1)\tilde r$, 
we have
\[
L \ll \log_2 (n/\tilde r +1) +1.
\]

\paragraph{Determinant.}
In \S\ref{sec-cholesky} we show the covariance matrix $\Sigma$ 
is the Schur complement of the expanded matrix.
For this expanded matrix, we also provide an explicit 
Cholesky factorization method
with linear time and space complexities.
We leverage this connection to argue that the determinant of $\Sigma$
equals to 
\[
\det(\Sigma) = \det(D)\prod_{l=1}^{L-1} \det(\Lambda_l).
\]
Therefore, $\det(\Sigma)$ can be computed at no additional cost while recursively 
computing  $\Sigma^{-1}$.
Moreover, Cholesky factors enable
feature-dependent linear transform that whitens the data
and offer multiple useful interpretations,
see \cite[\S2]{barratt2023covariance}.
See \S\ref{sec-chol-determinant} for details.

\subsection{EM iteration}
\subsubsection{Selection matrices}\label{sec-selection-matrices}
Let $s_i$ be the $i$th row sparsity pattern of $F$.
We denote by $|s_i|$ the number of rows that share this sparsity. 
Then the number of unique
sparsity patterns of rows of $F$ equals the number of groups
at level $L-1$, \ie, $p_{L-1}$.
Note that we must have $\sum_{i=1}^{p_{L-1}} |s_i|=n$.
Let $S_{r_i} \in \{0, 1\}^{|s_i| \times n}$ be a 
matrix that selects rows with $i$th sparsity pattern.
Since any row sparsity pattern of $F$ has $\sum_{l=1}^{L-1}r_l=r-1$ nonzero columns,
we define $S_{c_i}^T \in \{0, 1\}^{s \times (r-1)}$
as a matrix that selects those columns of $F$.
Thus, number of nonzero columns for row sparsity pattern $s_i$
is $r-1$, and the matrices 
\[
S_{r_i}F S_{c_i}^T \in \reals^{|s_i| \times (r-1)}, \quad i=1, \ldots, p_{L-1},
\]
are dense in the coefficients of $F$, see figure~\ref{fig-sparsity-F}.

\begin{figure}[H]
    \begin{center}    
    \includegraphics[width=0.4\textwidth]{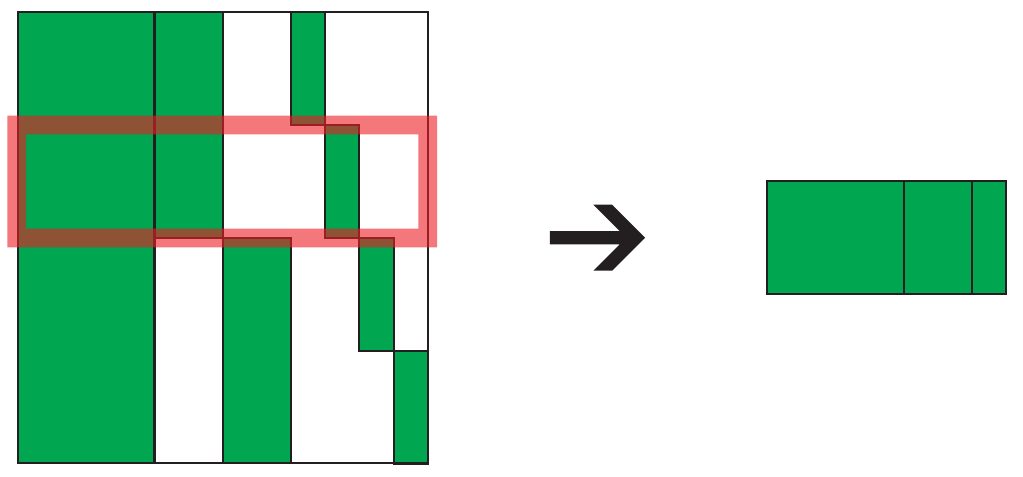}
    \end{center}
    \caption{Structured matrix $F$ with $p_{3}=4$ row sparsity patterns is 
    shown on the left. 
    The second row sparsity pattern is highlighted in red.
    The dense matrix $S_{r_2}F S_{c_2}^T$ is shown on the right.}
    \label{fig-sparsity-F}
\end{figure}

\begin{remark}
For any matrix $M$ with $s$ rows we have
\[
S_{r_i} F M = S_{r_i} F S_{c_i}^T S_{c_i} M, \quad i=1, \ldots, p_{L-1}.
\] 
\end{remark}

\subsubsection{EM iteration computation}\label{sec-em-computation}
Recall that $Q(F, D; F^0, D^0)$~\eqref{e-expect-step} is separable across 
the rows of $F$.
Therefore, to find $F^1$ we solve the reduced least squares problem 
for each sparsity pattern of $F$.

Recall matrices $V$ \eqref{e-matrix-V} and $W$ \eqref{e-matrix-W},
where $W\succ 0$.
To find the coefficients of $F$ in problem~\eqref{e-max-step},
using \S\ref{sec-selection-matrices}, 
it suffices to minimize the following 
\BEAS
\Tr (F W F^T - 2F V) &=& \sum_{i=1}^{p_{L-1}} \Tr \left(
S_{r_i}FW F^TS_{r_i}^T - 2S_{r_i}F VS_{r_i}^T\right) \\
&=& \sum_{i=1}^{p_{L-1}} \Tr \left((S_{r_i}F S_{c_i}^T) 
(S_{c_i} W S_{c_i}^T)(S_{r_i}F S_{c_i}^T)^T - 2(S_{r_i}F S_{c_i}^T) 
(S_{c_i}VS_{r_i}^T)\right).
\EEAS
To recover the coefficients of $F$, we solve the 
least squares problem, 
\BEQ\label{e-si-lin-solve-e-step}
S_{r_i}F S_{c_i}^T = (S_{c_i}VS_{r_i}^T)^T(S_{c_i} W S_{c_i}^T)^{-1}
\EEQ
for each $i=1, \ldots, p_L$. 
The inverse operation above is well-defined, since 
$W \succ 0$ implies $S_{c_i} W S_{c_i}^T \succ 0$.

We now derive the computational complexity for calculating $F^1$.
We first compute coefficients of MLR $(\Sigma^0)^{-1}$ in $T(n)$.

Next we describe how to efficiently compute 
$S_{c_i}VS_{r_i}^T$ and $S_{c_i} W S_{c_i}$.
Since $F^0 S_{c_i}^T \in \reals^{n \times (r-1)}$,
we compute $(\Sigma^0)^{-1}(F^0 S_{c_i}^T) \in \reals^{n \times (r-1)}$ in 
$O(nr^2)$ using \S\ref{s-inverse}.
Next we compute 
\[
\left ((S_{c_i}{F^0}^T)(\Sigma^0)^{-1}\right )(F^0 S_{c_i}^T)
\in \reals^{(r-1)\times (r-1)}
\]
in $O(nr^2)$.
To evaluate the product 
$\left ((S_{c_i}{F^0}^T)(\Sigma^0)^{-1}\right )Y^T\in \reals^{ (r-1) \times  N}$ 
we need $O(nrN)$.
Combining the above, we obtain 
\[
S_{c_i}VS_{r_i}^T = 
\left(S_{c_i}{F^0}^T (\Sigma^0)^{-1}Y^T\right)(Y S_{r_i}^T) 
\in \reals^{ (r-1) \times |s_i|}
\]
in $O(|s_i|rN)$. 
Also by computing
\[
\left((S_{c_i}{F^0}^T)(\Sigma^0)^{-1}Y^T\right) \left(Y (\Sigma^0)^{-1} F^0 
S_{c_i}^T\right )
\in \reals^{(r-1)\times (r-1)}
\]
in $O(r^2N)$, we then get $S_{c_i} W S_{c_i}^T\in \reals^{ (r-1) \times (r-1)}$
in $O(r^2)$.
Given $S_{c_i}VS_{r_i}^T$ and $S_{c_i} W S_{c_i}$, 
solving the linear system~\eqref{e-si-lin-solve-e-step} 
takes $O(|s_i|r^3)$.

When solving for each sparsity pattern $s_i$, 
the total complexity of the maximization step is
\[
T(n) + \sum_{i=1}^{p_{L-1}} O(nr^2+ nrN + |s_i|rN + r^2N 
+ |s_i|r^3),
\]
which simplifies to
\[
T(n) + O(p_{L-1}nr^2 + p_{L-1}nrN + p_{L-1}r^2N + nr^3).
\]
Plugging in the complexity of the inverse computation we arrive
at
\[
O(p_{L-1}nr^2 + nr^3 + p_{L-1}nrN + p_{L-1} r_{\max}r^2
+ p_{L-1}r^2N).
\]
Since $p_{L-1}\ll n$, the time complexity is linear in $n$.

As a stopping criteria we use the relative difference
between consecutive log-likelihoods of observations~\eqref{e-ll-observations}. 
This requires computing the determinant of the covariance matrix,
which we obtain at no cost during the inverse computation.
See \S\ref{sec-cholesky} and \S\ref{sec-chol-determinant} for details.

\section{Numerical examples}\label{sec-experiments}
We compare two factor fitting approaches 
based on Frobenius norm~\cite{parshakova2023factor} and MLE.
In the first example, we compare a traditional factor model (FM) with a 
multilevel factor model (MFM) using real data.
We demonstrate that the multilevel factor model significantly improves the 
likelihood of the observations. 
In the second example, we consider a synthetic multilevel factor model to 
generate the observations. 
Our results show that the expected log-likelihood 
distribution of the MLE-based method 
significantly outperforms the Frobenius norm-based method.
Finally, we apply our method to the real-world large-scale example.

\subsection{Asset covariance matrix}\label{sec-asset-cov}
We focus on the asset covariance matrix from~\cite[\S8.1]{parshakova2023factor}.
In this example the daily returns of $n=5000$ assets are found
or derived from data from CRSP Daily Stock and  CRSP/Compustat 
Merged Database
\textcopyright2023 Center for Research in Security Prices 
(CRSP$^\text{\textregistered}$), 
The University of Chicago Booth School of Business. 
We consider a $N=300$ 
(trading) day period ending 2022/12/30,  
and for hierarchical partition use
Global Industry Classification Standard (GICS)~\cite{bhojraj2003s}
codes from CRSP/Compustat Merged Database 
-- Security Monthly during 2022/06/30 to 2023/01/31
which has $L=6$ levels.

\begin{table}
\centering
\begin{tabular}{l c c c} 
 Fit & Model & ${\|\hat \Sigma- \Sigma\|_F}/{\|\Sigma\|_F}$ & $\ell(F,D; Y)/N$ \\  
 \hline
Frob  & FM &  $0.1538$ & $11809$\\
MLE  & FM &  $0.1617$ & ${11907}$\\
Frob & MFM & $0.1648$  & $11956$\\
MLE & MFM &  $0.8497$ & $\bf{12114}$\\
\end{tabular}
\caption{Frobenius errors and average log-likelihoods 
for factors fitted using either the Frobenius norm or 
MLE-based methods for the asset covariance matrix.}
\label{tab-cov-error-mfm}
\end{table}

We use the GICS hierarchy and two different rank allocations; 
see figure~\ref{fig-asset-cov-em} and table~\ref{tab-cov-error-mfm}.
For a rank allocation of $r_1=29,~r_2=\cdots=r_5=0,~r_6=1$
(\ie, a traditional factor model), 
our method's average log-likelihood of realized returns 
improves by $98$ compared to the
Frobenius norm-based method. 
Alternatively, using ranks $r_1=14,~r_2=6,~r_3=4,~r_4=3,~r_5=2,~r_6=1$,
as determined by the rank allocation algorithm in \cite{parshakova2023factor}
(\ie, multilevel factor model), 
the average log-likelihood increases by $158$.
Thus the best log-likelihood is achieved using the multilevel factor model
fitted with MLE-based objective.
Also note that a low Frobenius error does not necessarily 
indicate a better log-likelihood, see table~\ref{tab-cov-error-mfm}.
% There are $s=697$ unique factors.

To assess whether the log-likelihoods of the two methods are significantly 
different,
we can compare it to the standard deviation of the expectation of these
log-likelihoods with respect to the true model. 
Since we do not have the density of the true model,
we assume that the samples are drawn 
from~\eqref{e-hier-factor-model}. Under this assumption 
the standard deviation of the average log-likelihood is $2.887$, 
see \S\ref{a-heuristic-variance}.
Therefore, we conclude that the log-likelihood for our method
MLE is significantly better.

\begin{figure}
    \begin{center}
    \includegraphics[width=\textwidth]{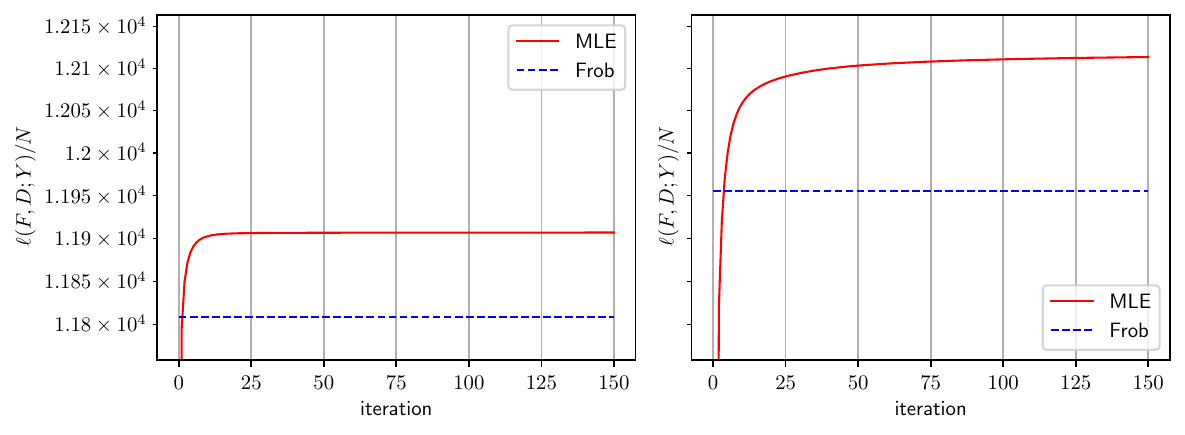}
    \end{center}
    \caption{Log-likelihood during the
    EM algorithm (red curve) and after Frobenius norm fitting (blue curve),
    for FM (left) and
    MFM (right) of the asset covariance matrix.}
    \label{fig-asset-cov-em}
\end{figure}

\subsection{Synthetic multilevel factor model}
We generate samples from a synthetic multilevel factor model 
with $n=10,000$ features. 
We create a random hierarchical partition with $L=6$.
Starting with a single group, we evenly divide it across levels, 
resulting in $4$, $8$, $16$, $32$, and finally $10,000$ groups at the bottom level. 
Each level is assigned ranks: $r_1=10,~r_2=5,~r_3=4,~r_4=3,~r_5=2,~r_6=1$, 
respectively, yielding $s=174$ unique factors in total.
The resulting compression ratio is $200:1$.

Following this, the coefficients of the structured factor matrix $F$
are sampled from $\mathcal{N}(0, 1)$.
Then we sample the noise variance in proportion to the average
signal variance maintaining a signal-to-noise ratio (SNR) of $4$. 
This is achieved by sampling $D_{ii}$ uniformly from the interval
\[
[0, 2 (\ones^T \diag(FF^T)/n)/\text{SNR}], \quad i=1, \ldots, n.
\]

To evaluate how effectively we can fit the factors using MLE,
we use the rank allocation and hierarchical partition from the true
model.
The model is fitted with $N=80$ samples and 
evaluated using expected log-likelihood (based on the density of the true
model).

Since in this example we have access to the true model
$\Sigma^{\text{true}}= F^{\text{true}}{F^{\text{true}}}^T 
+ D^{\text{true}}$,
we can compute the expected log-likelihood
\[
\Expect(\ell(F, D; y))
=- \frac{n}{2} \log(2\pi) - \frac{1}{2}\log \det (FF^T + D) - 
\frac{1}{2}\Tr((FF^T + D)^{-1} \Sigma^{\text{true}}).
\]

We compare the average log-likelihood of two fitting approaches 
based on Frobenius norm and MLE; 
see figure~\ref{fig-synthetic-em} and table~\ref{tab-synthetic-example}.
Our method outperforms the Frobenius norm-based approach, 
showing a $284$ higher average log-likelihood on the sampled data $Y$
and a $372$ greater expected log-likelihood.

We generate $200$ samples $Y$, and for each $Y$, fit the model with two 
competing methods.
The resulting histograms of expected log-likelihoods $\Expect(\ell(F,D; y))$
are shown on figure~\ref{fig-synthetic-hist}.
The histogram of differences
$\Expect(\ell(F^{
\text{MLE}
},D^{
\text{MLE}
}; y)) - \Expect(\ell(F^{
\text{Frob}
},D^{
\text{Frob}
}; y))$
is displayed on figure~\ref{fig-synthetic-hist-diff}.
The mean of the differences is $371$, with a standard deviation of $136$, 
and for $99.5\%$ of the samples, the difference is positive.
Based on these histograms, we conclude that the distribution of the MLE-based 
method 
is significantly better than that of 
Frobenius norm-based method.

% FR: train ll=-20863.80491273808, exp ll=-24870.186339392043
% ML: train ll=-20624.63454962665, exp ll=-24551.649327337775
% TR: train ll=-22031.0012239113, exp ll=-22068.491261055737
\begin{table}
\centering
\begin{tabular}{l c c c} 
 Fit &  $\ell(F,D; Y)/N$ &  $\Expect (\ell(F,D; y))$ \\  
 \hline
Frob  &  $-20851$ & $-24843$\\
MLE   &  $-20567$ & $-24471$\\
True   &  $-22031$ & $-22068$\\
\end{tabular}
\caption{Log-likelihood 
for models fitted using the Frobenius norm,
MLE-based methods and the true model for a single $Y$ in the synthetic example.}
\label{tab-synthetic-example}
\end{table}

\begin{figure}
    \begin{center}
    \includegraphics[width=0.6\textwidth]{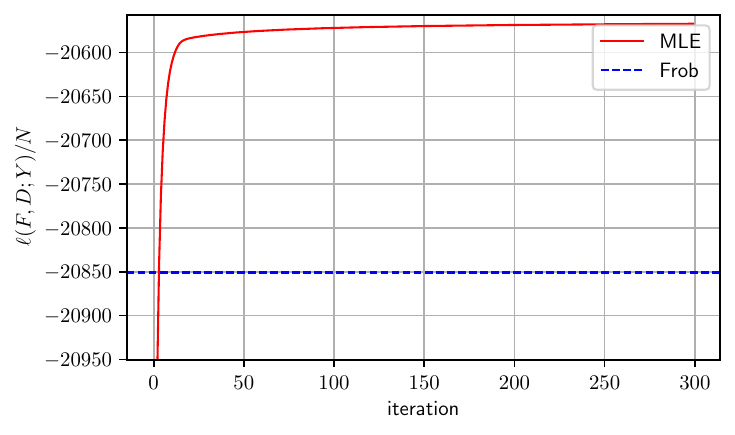}
    \end{center}
    \caption{Log-likelihood during the
    EM algorithm (red curve) and after Frobenius norm fitting (blue curve)
    for a single $Y$ in the synthetic example.}
    \label{fig-synthetic-em}
\end{figure}

\begin{figure}
    \begin{center}
    \includegraphics[width=0.7\textwidth]
    {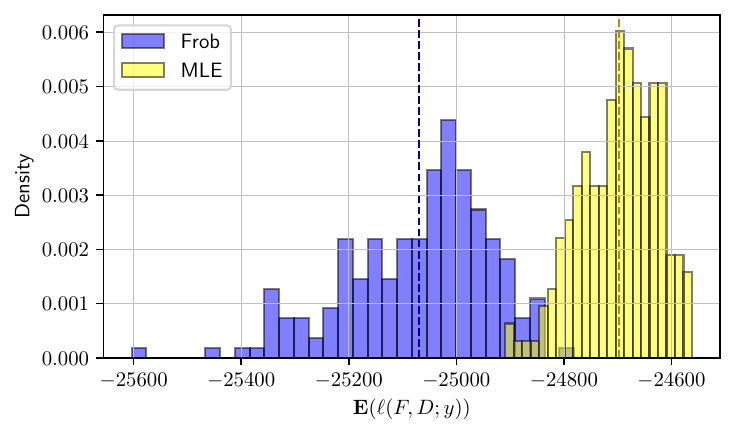}
    \end{center}
    \caption{Histograms of expected log-likelihoods
    for MLE and Frobenius norm-based fitting methods.}
    \label{fig-synthetic-hist}
\end{figure}

\begin{figure}
    \begin{center}
    \includegraphics[width=0.6\textwidth]
    {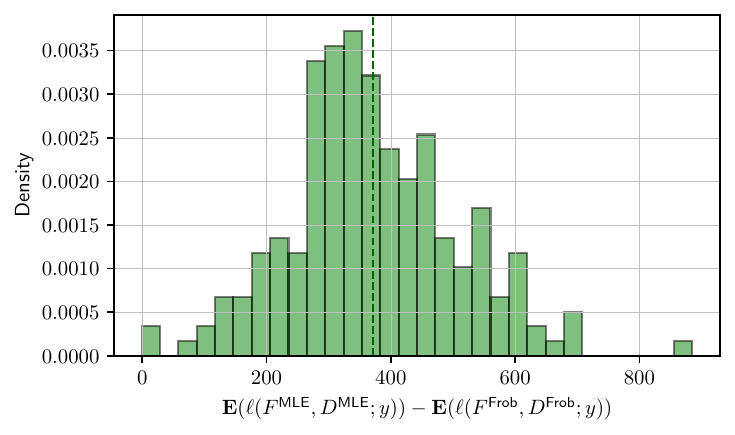}
    \end{center}
    \caption{Histogram of differences in expected log-likelihoods
    between MLE and Frobenius norm-based fitting methods.}
    \label{fig-synthetic-hist-diff}
\end{figure}

\subsection{Large-scale single-cell RNA sequencing dataset}
Single-cell RNA sequencing generates transcript count matrices that contain
gene expression profiles of individual cells.
In this section we use the dataset from \cite{dominguez2022cross, czicellxgene_explorer_2021, czi2025cz},
that contains immune cells from human tissues.

The original dataset contains $329,762$ cells with $36,398$ genes, 
collected from $12$ donors. 
Then we follow standard 
preprocessing steps for single-cell RNA sequencing data \cite{mccarthy2017scater, levine2023cell2sentence}.
We use Scanpy package \cite{wolf2018scanpy} for quality control metrics
\cite{mccarthy2017scater} to filter out low-quality cells and uninformative genes.
In particular, we filter cells with fewer than $200$ genes and filter 
genes expressed in fewer than $200$ cells.
We also filter out cells with more than $20\%$ of transcript counts from mitochondrial genes,
or which contain more than $2,500$ detected gene types.
Next, we normalize gene counts per cell, and subsequently apply log-plus-one 
transformation.
To reduce the dimensionality, we selected the top $500$ most variable genes.
The final feature matrix is standardized across cells 
and has $n=280,535$ cells and $N=500$ genes.

We use a hierarchy with $L=3$ levels,
grouping level $l=2$ by donor IDs (\ie, making $12$ groups in $l=2$).
We set the rank allocation to $r_1=12,~r_2=8,~r_3=1$. 
Our method achieves an average log-likelihood of $-217,730$, which
is by $4376$ larger than
the Frobenius norm-based method with $-222,106$, see figure~\ref{fig-scrna-em}.

In this experiment we expect the $r_2=8$ factors on level $l=2$ to capture donor-specific 
correlations, while the factors on level $l=1$ are to be shared across all the donors
and to describe the correlations across the cells.
In figure \ref{fig-scrna-celltype} we plot the factor loadings $F_1$, reordered to
display the cell types as contiguous groups, using CellTypist labels \cite{dominguez2022cross}.
The horizontal yellow lines indicate the ranges of the cell types.
We can see that some factors are strong predictors for specific cell types.
For instance, the second factor (the second column) predicts B cells with large 
positive loadings,
and both CD16+ and CD16- NK cells with large negative loadings.
Similarly, the fifth factor is associated with
classical monocytes and macrophages through large positive loadings.
Furthermore, $r_1=12$ factors on the first level explain on average 
$68\%$ of their individual variances.
In contrast, applying our fitting method to
the factor model, \ie, flat hierarchy with $L=2$ and
and $12$ factors, results in factors that explain on average 
$62\%$ of their individual variances.

\begin{figure}
    \begin{center}
    \includegraphics[width=0.6\textwidth]{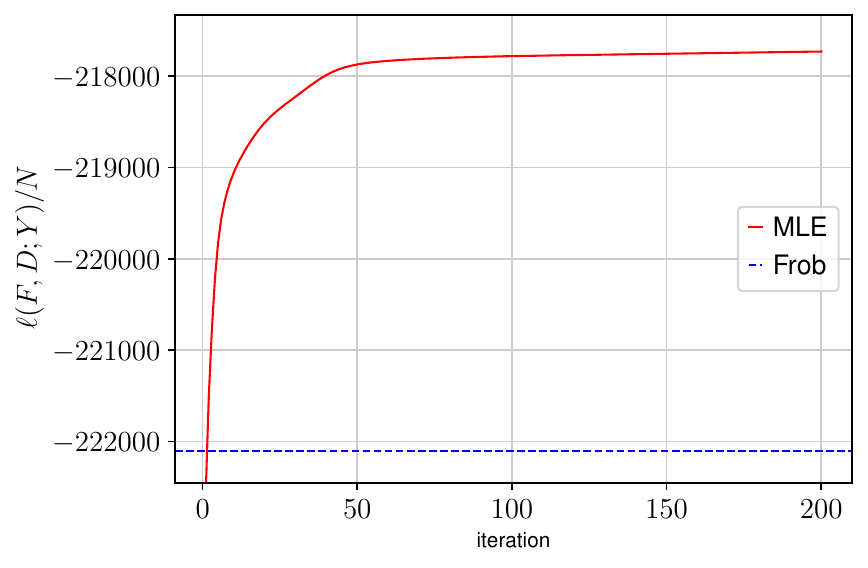}
    \end{center}
    \caption{Log-likelihood during the
    EM algorithm (red curve) and after Frobenius norm fitting (blue curve)
     in the single-cell RNA example.}
    \label{fig-scrna-em}
\end{figure}

\begin{figure}
    \begin{center}
    \includegraphics[width=0.7\textwidth]{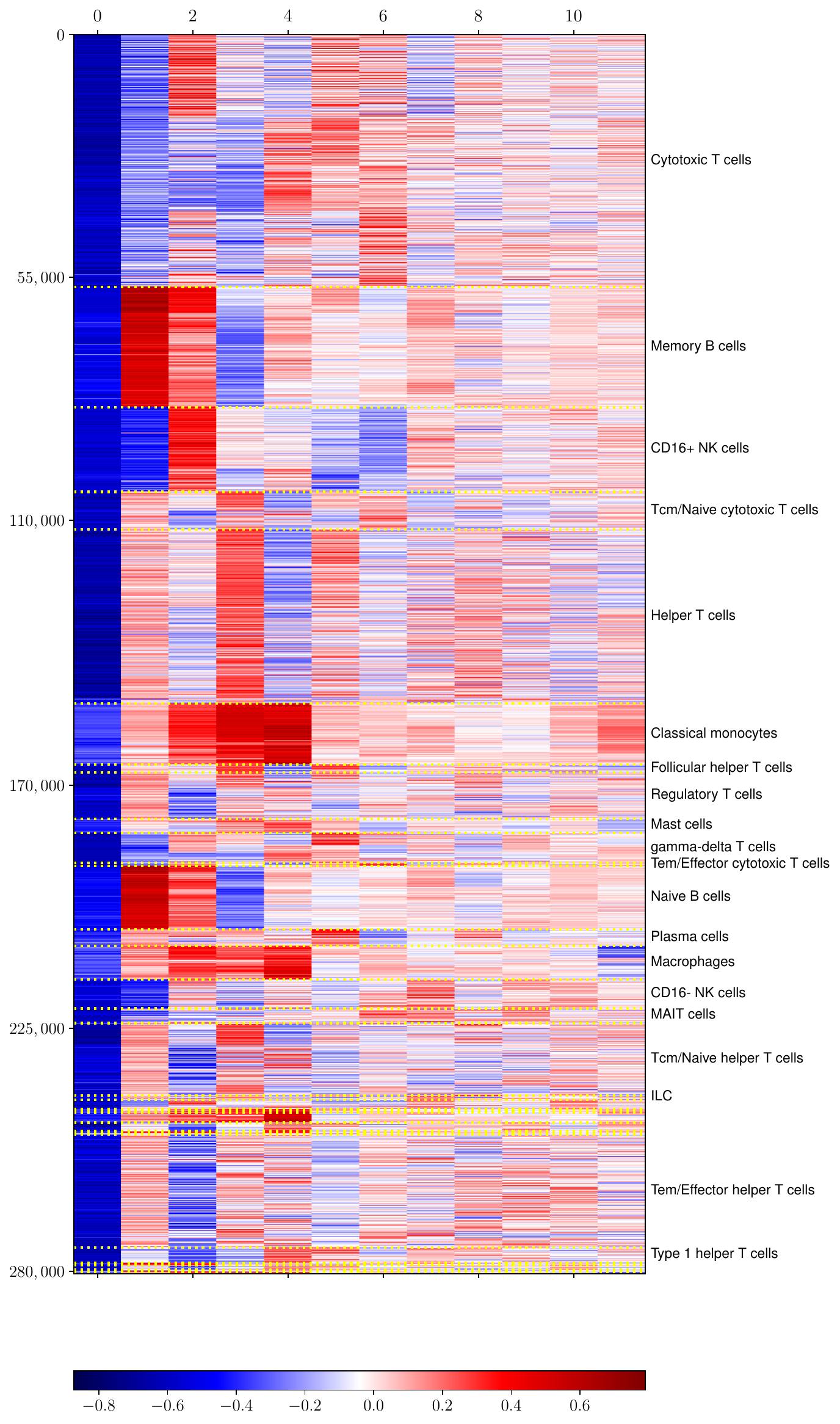}
    \end{center}
    \caption{Factor loading matrix $F_1 \in \reals^{n \times r_1}$ with 
    reordered rows to
    display the cell types as contiguous groups
    in the single-cell RNA example.}
    \label{fig-scrna-celltype}
\end{figure}

\section{Conclusion}
In this work, we present a novel and computationally efficient algorithm for
fitting multilevel factor model.
We introduce a fast implementation of the EM algorithm that uses linear time and 
space complexities per iteration, making it scalable.
This method relies on a novel fast algorithm for computing the inverse 
and a determinant of the PSD MLR matrix.

We also provide an open-source implementation of our methods,
that demonstrate their effectiveness on several examples,  
including the large-scale real-world example.
Our MLE-based method consistently outperforms the 
Frobenius norm-based method.
In this paper we assume that the hierarchy as well as rank allocation are known. 
Future research will focus on developing scalable heuristics for finding 
hierarchy and rank allocation while leveraging our fast factor-fitting
method.
The challenge is to keep storage and time complexities nearly linear.
As demonstrated in \S\ref{a-second-approx-ll}, minimizing
the Frobenius norm-based error approximately maximizes the log-likelihood.
Therefore, one promising approach is to adapt the techniques from \cite{parshakova2023factor}
to avoid forming dense matrices and store all matrices in the factored form.
For example, applying partial singular value decomposition to the matrices
in factored form,
enables the rank exchange algorithm to be applied straightforwardly to our setting.
However, the incremental hierarchy construction 
% (level by level from the top to the bottom) 
is less straightforward to apply as it requires forming 
dense residual matrices.
Specifically, it is based on the nested spectral dissection \cite{parshakova2023factor},
which involves computation of the second smallest eigenvalue of
Laplacian matrix for graph with adjacency matrix given by the
squared valued in the residual matrix.
And even though the residual matrix is factored, 
when we square it elementwise, this structure changes.
Future work will focus on the development of spectral clustering methods for
the factored residual matrix, while respecting the storage requirements.

% \clearpage
% \subsection*{Acknowledgments}

\clearpage
\bibliography{references}
\clearpage

\appendix

\section{Second order approximation of log-likelihood}\label{a-second-approx-ll}
In this section we explain the intricate relationship 
between Frobenius norm and MLE-based losses.
Let $S = Y^TY/N$ be a sample covariance matrix.
Then the average log-likelihood of $N$ data points for a
Gaussian model $y\sim N(0, \Sigma)$ is
\[
\frac{1}{N}\ell(\Sigma; Y) 
=- \frac{n}{2} \log(2\pi) - \frac{1}{2}\log \det \Sigma - 
\frac{1}{2}\Tr(\Sigma^{-1} S).
\]

We now derive the second-order approximation of the average log-likelihood.
We start with finding the second-order approximation of the function 
$f:\symm^n \to \reals$,
\[
f(\Sigma) = \log \det \Sigma, \quad \dom f = \symm^n_{++}.
\]
Following the derivation of~\cite[\S A.4]{boyd2004convex}, 
let $\Delta \in \symm^n$ be such that $(\Sigma+\Delta) \in \symm^n_{++}$ 
is close to $\Sigma$.
We have 
\BEAS 
\log \det (\Sigma+\Delta) &=& \log \det \left( \Sigma^{1/2}(I + \Sigma^{-1/2}
\Delta \Sigma^{-1/2})\Sigma^{1/2} \right ) \\
&=& \log \det \Sigma  + \log \det(I + \Sigma^{-1/2}\Delta \Sigma^{-1/2}) \\
&=& \log \det \Sigma  + \sum_{i=1}^1 \log (1 + \lambda_i),
\EEAS 
where $\lambda_i$ is the $i$th eigenvalue of $\Sigma^{-1/2}\Delta \Sigma^{-1/2}$.
Since $\Delta$ is small, then $\lambda_i$ are small.
Thus to second order, we have
\[
\log (1 + \lambda_i) \approx \lambda_i - \frac{\lambda_i^2}{2}.
\]
Combining the above we get
\[
\log \det (\Sigma+\Delta) - \log \det \Sigma  \approx  \sum_{i=1}^n 
(\lambda_i - \frac{\lambda_i^2}{2}) = \Tr(\Sigma^{-1}\Delta) - 
\frac{1}{2}\Tr\left ( \Sigma^{-1}\Delta \Sigma^{-1} \Delta \right ). 
\] 
We used the fact the sum of eigenvalues is the trace, 
and the eigenvalues of the product of a symmetric matrix with itself are 
the squares of the eigenvalues of the original matrix,
and the cyclic property of trace.

Next we find the second-order approximation of the function
$g:\symm^n \to \reals$,
\[
g(\Sigma) = \Tr(\Sigma^{-1} S), \quad \dom g = \symm^n_{++}.
\]
Since $\Sigma \succ 0$, we have
\[
\Tr( (\Sigma+\Delta)^{-1} S) = \Tr \left( \Sigma^{-1/2}(I + 
\Sigma^{-1/2}\Delta \Sigma^{-1/2})^{-1}\Sigma^{-1/2} S \right ).
\]
Recall $\Delta \in \symm^n$ is small, therefore the 
spectral radius of $\Sigma^{-1/2}\Delta \Sigma^{-1/2}$ is smaller than $1$.
Thus using the Neuman series to second order we have
\[
(I + \Sigma^{-1/2}\Delta \Sigma^{-1/2})^{-1} \approx 
I - \Sigma^{-1/2}\Delta \Sigma^{-1/2} + \Sigma^{-1/2}\Delta 
\Sigma^{-1}\Delta \Sigma^{-1/2}.
\]
Combining the above we get
\[
\Tr( (\Sigma+\Delta)^{-1} S) \approx \Tr( \Sigma^{-1}S) - 
\Tr(\Sigma^{-1}\Delta \Sigma^{-1}S) + \Tr\left ( \Sigma^{-1}\Delta \Sigma^{-1} 
\Delta \Sigma^{-1} S\right ).
\]

Using the above derivations, the second-order approximation 
of the average log-likelihood
at $S$ is the quadratic function of $\Sigma$ given by
\BEQ
\frac{1}{N}\ell(\Sigma; Y) 
\approx \frac{1}{N}\ell(S; Y) - \frac{1}{4}\|S^{-1} (S - 
\Sigma)\|_F^2 
= \frac{1}{N}\ell(S; Y) - \frac{1}{4}\|I - S^{-1}\Sigma\|_F^2. 
\label{e-frob-ll-relation}
\EEQ
Finally, \eqref{e-frob-ll-relation} gives the relationship between 
the log-likelihood and Frobenius norm.

\section{Heuristic method for variance estimation}\label{a-heuristic-variance}
In \S\ref{sec-experiments}, we compare the log-likelihoods 
of models fitted using Frobenius-based loss or MLE.
To assess if the difference in the log-likelihoods is significant, 
we present a heuristic method for estimating the variance
of the average log-likelihood.
We assume that the empirical data is coming 
from model~\eqref{e-hier-factor-model} with
parameters $F$ and $D$.
Then the average log-likelihood of $N$ data points is
\BEAS
\frac{1}{N}\ell(F, D; Y) &=& 
- \frac{n}{2} \log(2\pi) - \frac{1}{2}\log \det(FF^T+D) - 
\frac{1}{2N}\Tr((FF^T+D)^{-1} Y^T Y) \\
&=&- \frac{n}{2} \log(2\pi) - \frac{1}{2}\log \det(\Sigma) - 
\frac{1}{2N}\sum_{i=1}^N y_i^T\Sigma^{-1}y_i.
\EEAS
Since $y_i \sim \mathcal{N}(0, \Sigma)$, then 
$\Sigma^{-1/2}y_i \sim \mathcal{N}(0, I)$.
This implies
\[
y_i^T\Sigma^{-1}y_i = (\Sigma^{-1/2}y_i)^T(\Sigma^{-1/2}y_i)
\sim \chi^2(n).
\]
Let $z_i = \Sigma^{-1/2}y_i$, thus
\[
\Var\left(\frac{1}{N}\ell(F, D; Y)\right)
= \Var\left(\frac{1}{2N}\sum_{i=1}^N z_i^T z_i\right)
= \frac{1}{4N^2}\sum_{i=1}^N\Var\left(z_i^T z_i\right) = 
\frac{n}{2N}.
\]
Also the expectation is
\BEAS
\Expect\left(\frac{1}{N}\ell(F, D; Y)\right)
&=& - \frac{n}{2} \log(2\pi) - \frac{1}{2}\log \det(\Sigma) -
\frac{1}{2N}\sum_{i=1}^N\Expect\left( z_i^T z_i\right)\\
&=& - \frac{n}{2} \log(2\pi) - \frac{1}{2}\log \det(\Sigma) -
\frac{n}{2}.
\EEAS

In the asset covariance example,
we have $n=5000$ and $N=300$. Therefore, the approximation 
to the standard deviation is 
\[
\sqrt{\frac{n}{2N}} \approx 2.887, 
\]
and of the expectation is
\[
- \frac{n}{2} (1 + \log(2\pi)) - \frac{1}{2}\log \det(\Sigma) -
\frac{n}{2} \approx -7095 - \frac{1}{2}\log \det(\Sigma).
\]

\section{Auxiliary derivations}

\begin{lemma}\label{lem-blockdiag-prod}
Let $B \in \reals^{m \times n}$ be a block diagonal matrix with 
block sizes determined by row and column index partitions $I$ and $J$, respectively.
Similarly, let $C \in \reals^{n \times \tilde n}$ be a block diagonal matrix with 
row and column index partitions $\tilde I$ 
and $\tilde J$, respectively.
If $J \preceq \tilde I$, 
then product $BC$ is also a block diagonal matrix with column sparsity 
given by the partition $\tilde J$.
Moreover, if $I = J$, then $\supp(BC) = \supp(C)$.
\end{lemma}
\begin{proof}
Define matrices in terms of blocks explicitly as
\[
B= \blkdiag(B_1, \ldots, B_p)  \in \reals^{m \times n},
\qquad 
C = \blkdiag(C_1, \ldots, C_{\tilde p}) \in \reals^{n \times \tilde n}.
\]
Similarly define index partitions as 
\[
\{ b^J_1, \ldots, b^I_p \} = I, 
\quad \{ b^J_1, \ldots, b^J_p \} = J, \quad
\{ \tilde b^{\tilde I}_1, \ldots, \tilde b^{\tilde I}_{\tilde p} \} = \tilde I, 
\quad \{ \tilde b^{\tilde J}_1, \ldots, \tilde b^{\tilde J}_{\tilde p} \} = \tilde J.
\]

For each $\tilde k=1, \ldots, \tilde p$, 
index set $\tilde b_{\tilde k}^{\tilde I} \in \tilde I$ 
is refined in $J$,
because $J \preceq \tilde I$.
Formally,
there exist some indices $0 \leq k_1 < k_2 \leq p$ such that  
\[
\bigcup_{k'=k_1}^{k_2}b^J_{k'} = \tilde b_{\tilde k}^{\tilde I},
% \quad b_{k_1}, \ldots, b_{k_2} \in J,
\quad b_0^J = \emptyset.
\]
Hence, the product $BC$ restricted to the rows indexed by
$\bigcup_{k'=k_1}^{k_2}b^I_{k'}$
is nonzero only in the columns indexed by $\tilde b_{\tilde k}^{\tilde I}$.

Therefore, $BC$ is block diagonal with $\tilde p$ blocks, where $\tilde k$th
block has size $|\bigcup_{k'=k_1}^{k_2}b^I_{k'}| \times |\tilde{ b}^{\tilde{J}}_{\tilde k}|$
and is given by
\[
\blkdiag(B_{k_1}, \ldots, B_{k_2}) C_{\tilde k},
\]
and $\tilde J$ defines its column partition.

If $I=J$, then 
\[
\bigcup_{k'=k_1}^{k_2}b^I_{k'} = \bigcup_{k'=k_1}^{k_2}b^J_{k'} = \tilde b_{\tilde k}^{\tilde I}.
\]
Therefore, for all $\tilde k=1, \ldots, \tilde p$ we have
\[
\supp(\blkdiag(B_{k_1}, \ldots, B_{k_2}) C_{\tilde k}) = \supp(C_{\tilde k}),
\]
which implies $\supp(BC) = \supp(C)$.
\end{proof}

\begin{lemma}\label{lem-Fls-prod}
Let $F \in \reals^{n \times pr}$ be a block diagonal matrix with 
$p$ blocks of size $n_k \times r$ for all $k=1, \ldots, p$.
Similarly, let $\tilde F \in \reals^{n \times \tilde p \tilde r}$ be a block 
diagonal matrix with 
$\tilde p$ blocks of size $\tilde n_k \times \tilde r$ for all $k=1, \ldots,\tilde p$.
If $\cI(F) \preceq \cI(\tilde F)$,
then product $\tilde F^T F$ is also a block diagonal matrix with $\tilde p$ blocks and
$p r \tilde r$ nonzero elements. Moreover, computing $\tilde F^T F$ takes 
$O(nr \tilde r)$.
\end{lemma}
\begin{proof}
Applying Lemma \ref{lem-blockdiag-prod}, $\cI(\tilde F^T F) = \cJ(\tilde F)$.
In other words $\tilde F^T F$ is a block diagonal matrix with $\tilde p$ blocks.

Consider any group $k$ in the partition $\cI(\tilde F)$, it is refined 
into $c_k \geq 1$ groups in $\cI(F)$. Then the diagonal block corresponding to group $k$
in $\cI(\tilde F)$ of size $\tilde n_k \times \tilde r$ interacts with 
$c_k$ respective block diagonal elements in $\cI(F)$,
forming a block diagonal matrix of size $\tilde n_k \times c_k r$.
This  block diagonal matrix 
has $c_k$ blocks with column index partition
\[
\{\{1, \ldots, r\}, \{r+1, \ldots, 2r\}, \ldots, \{(c_k-1)r, \ldots, c_kr\}\}.
\]
Thus the matrix-vector multiplication with this matrix
requires $O(\tilde n_k r)$ operations.
Therefore, computing $\tilde F^T F$ requires the order of
$\sum_{k=1}^{\tilde p} \tilde n_k r \tilde r = n r \tilde r$
operations.
Finally, the number of nonzero elements in $\tilde F^T F$ is
$\sum_{k=1}^{\tilde p} c_k r \tilde r = p r \tilde r$.

\end{proof}

\subsection{EM method}\label{apx-e-step}
This section complements \S\ref{sec-e-step}.
Using the joint distribution $(y,z)$ and conditional distribution $z_i \mid y_i, F^0, D^0$
defined in \S\ref{sec-e-step}, we get
\BEAS 
\sum_{i=1}^N\Expect \left ( z_i z_i^T \mid y_i, F^0, D^0 \right ) &=& \sum_{i=1}^N 
\Cov((z_i, z_i) \mid y_i, F^0, D^0)) \\
&&+ \Expect \left ( z_i \mid y_i, F^0, D^0 \right )\Expect 
\left ( z_i \mid y_i, F^0, D^0 \right )^T\\
&=& N(I_{s} - 
{F^0}^T(\Sigma^0)^{-1}F^0) 
+ {F^0}^T(\Sigma^0)^{-1}Y^TY (\Sigma^0)^{-1} F^0.
\EEAS

We can now derive the expression for \eqref{e-q-e-step}
\BEAS
Q(F, D; F^0, D^0) &=& \Expect \left ( \ell(F, D; Y, Z)  \mid Y, F^0, D^0  \right ) \\
% &=& -  \frac{(n+s)N}{2}\log(2\pi) -\frac{N}{2}\log \det D  \\ 
% &&- \frac{1}{2} 
% \sum_{i=1}^N \Expect \left ( \Tr (D^{-1}(y_i - Fz_i)(y_i - Fz_i)^T) + \Tr (z_iz_i^T)
% \mid Y, F^0, D^0  \right ) \\
&=& -  \frac{(n+s)N}{2}\log(2\pi) -\frac{N}{2}\log \det D 
- \frac{1}{2} 
\sum_{i=1}^N \Tr \big(\Expect \left (z_iz_i^T \mid y_i, F^0, D^0  \right ) \big)\\
&&- \frac{1}{2} 
\sum_{i=1}^N \Tr  \big(D^{-1} \{  
(y_i y_i^T - 2F \Expect \left ( z_i \mid y_i, F^0, D^0  \right ) y_i^T)  \\
&&+  F \Expect \left ( z_iz_i^T \mid y_i, F^0, D^0  \right ) F^T
\} \big)\\
&=& -  \frac{(n+s)N}{2}\log(2\pi) -\frac{N}{2}\log \det D \\
&&- \frac{1}{2} 
\Tr \big( \underbrace{N(I_{s} - 
{F^0}^T(\Sigma^0)^{-1}F^0) 
+ {F^0}^T(\Sigma^0)^{-1}Y^TY (\Sigma^0)^{-1} F^0}_{=W}\big)\\
&&- \frac{1}{2} 
\Tr  \Big(D^{-1} \big \{  
Y^TY - 2F \underbrace{{F^0}^T (\Sigma^0)^{-1}Y^TY}_{=V}  \\
&&+  F (\underbrace{N(I_{s} - 
{F^0}^T(\Sigma^0)^{-1}F^0) 
+ {F^0}^T(\Sigma^0)^{-1}Y^TY (\Sigma^0)^{-1} F^0}_{=W}) F^T
\big \} \Big). 
\EEAS

\subsection{Inverse computation}\label{apx-inv-comp}
SMW matrix identity implies 
% \eqref{e-inverse-woodbury} 
\BEAS 
(F_{l+} F_{l+}^T + D)^{-1} 
&=& 
(F_{(l+1)+} F_{(l+1)+}^T + D)^{-1} 
- \underbrace{(F_{(l+1)+} F_{(l+1)+}^T + D)^{-1}F_l}_{=M_0}\nonumber\\ 
&& 
(I_{p_lr_l} + 
F_l^T\underbrace{(F_{(l+1)+}F_{(l+1)+}^T + D)^{-1} F_l}_{=M_0})^{-1} 
\underbrace{F_l^T(F_{(l+1)+}F_{(l+1)+}^T + D)^{-1}}_{=M_0^T} \nonumber\\
&=& (F_{(l+1)+} F_{(l+1)+}^T + D)^{-1} -  M_0(I_{p_lr_l} + 
F_l^TM_0)^{-1} M_0^T \\
&=&(F_{(l+1)+} F_{(l+1)+}^T + D)^{-1} - H_l H_l^T.
\EEAS

Therefore, we have
\[
\Sigma_{l+}^{-1} = \Sigma_{(l+1)+}^{-1} - H_l H_l^T.
\]

\section{Cholesky factorization}\label{sec-cholesky}
In this section we present a Cholesky factorization for the expanded matrix
and show that Cholesky factor has the same sparsity as its inverse.

\subsection{Schur complement}
Finding the inverse of $\Sigma$ 
amounts to 
solving the linear system
\[
(FF^T + D)X = D X + F_{L-1}F_{L-1}^TX + \cdots + F_1F_1^T X = I_n,
\]
% \BEAS
% (FF^T + D)X &=& D X + F_{L-1}F_{L-1}^TX + \cdots + F_1F_1^T X \\
% &=& D X + F_{L-1}Y_{L-1} + \cdots + F_1Y_1 \\
% &=& I_n,
% \EEAS
which is equivalent to solving expanded system of equations
\BEQ\label{e-chol-exp-sys1}
\begin{bmatrix}
D &  F_{L-1} &  \cdots & F_1 \\
F_{L-1}^T & -I_{p_{L-1}r_{L-1}} \\
\vdots & & \ddots & \\
F_1^T &  &  & -I_{p_1r_1}
\end{bmatrix} 
\begin{bmatrix}
X \\
Y_{L-1} \\
\vdots \\
Y_1
\end{bmatrix} = 
\begin{bmatrix}
I_n \\
\zeros
\end{bmatrix}.
\EEQ
Denote the expanded matrix~\eqref{e-chol-exp-sys1} by $E\in \symm^{n+s}$.
Note that $E$ has the block sparsity pattern of
the upward-left arrow.
 
Block Gaussian elimination on the matrix~\eqref{e-chol-exp-sys1}
leads to an LDL decomposition
\BEQ\label{e-chol-ldu}
E
=
% \begin{bmatrix}
% I_n &  -F \\
%     & I_s
% \end{bmatrix} 
\begin{bmatrix}
I_n &  -F_{L-1} &  \cdots & -F_1 \\
 & I_{p_{L-1}r_{L-1}} \\
 & & \ddots & \\
 &  &  & I_{p_1r_1}
\end{bmatrix} 
\begin{bmatrix}
FF^T+D &   \\
 & -I_s 
\end{bmatrix}
% \begin{bmatrix}
% I_n &   \\
% -F^T & I_s 
% \end{bmatrix}
\begin{bmatrix}
I_n &  \\
-F_{L-1}^T & I_{p_{L-1}r_{L-1}} \\
\vdots & & \ddots & \\
-F_1^T &  &  & I_{p_1r_1}
\end{bmatrix}.
\EEQ
And $FF^T+D$ is Schur complement of the block $-I_s$ of the matrix $E$.

\subsection{Recursive Cholesky factorization}\label{sec-rec-chol-factor}
Let $s_{l+}=\sum_{l'=l}^{L-1} p_{l'}r_{l'}$ for all $l=1, \ldots, L-1$.
Define $E_l$ as the top left $(n+s_{l+})\times (n+s_{l+})$ submatrix of $E$, \ie, 
\[
E_l = \begin{bmatrix}
D &  F_{L-1} &  \cdots & F_l \\
F_{L-1}^T & -I_{p_{L-1}r_{L-1}} \\
\vdots & & \ddots & \\
F_l^T &  &  & -I_{p_lr_l}
\end{bmatrix} \in \symm^{n+s_{l+}}.
\]
We find the factors of $E$ 
by recursively factorizing $E_{L-1}, \ldots, E_1$
using the relation
\[
E_l = \begin{bmatrix}
E_{l+1} &
\begin{bmatrix}
F_{l}\\ 
\zeros
\end{bmatrix}\\
\begin{bmatrix}
F_{l}^T & \zeros 
\end{bmatrix}
& -I_{p_lr_l}
\end{bmatrix}. 
\]

\subsubsection{Sparsity patterns}

The block Gaussian elimination on $E_l$ 
gives the following factorization
\BEQ 
\begin{bmatrix}
I_{n+s_{(l+1)+}} \\
\begin{bmatrix}
F_{l}^T & \zeros
\end{bmatrix} E_{l+1}^{-1} & I_{p_lr_l}
\end{bmatrix}
\begin{bmatrix}
E_{l+1} \\
& 
-\left(I_{p_lr_l}+\begin{bmatrix}
F_l^T & \zeros
\end{bmatrix}E_{l+1}^{-1}
\begin{bmatrix}
F_l \\ \zeros
\end{bmatrix}
\right )
\end{bmatrix}
\begin{bmatrix}
I_{n+s_{(l+1)+}} \\
\begin{bmatrix}
F_{l}^T & \zeros
\end{bmatrix} E_{l+1}^{-1} & I_{p_lr_l}
\end{bmatrix}^T. \label{e-cholseky-el-elp1}
\EEQ

\paragraph{Submatrices of $E$.}
In Lemma \ref{lem-sparsity-E} we show the sparsity pattern of matrices 
necessary for 
Cholesky factorization.

\begin{lemma}\label{lem-sparsity-E}
Let $F$ and $D$ be factors of PSD MLR $\Sigma$,
and $E$ be its expanded matrix.
Then for all $l=1, \ldots, L-1$, we have
\[
\begin{bmatrix}
F_{(l-1)-}^T & \zeros
\end{bmatrix}E_l^{-1}
\begin{bmatrix}
I_n \\ \zeros
\end{bmatrix}
=F_{(l-1)-}^T\Sigma_{l+}^{-1}, 
\]
and 
\[
\supp(\begin{bmatrix}
F_{(l-1)-}^T & \zeros
\end{bmatrix}E_{l}^{-1}) = \supp(
F_{(l-1)-}^T 
\begin{bmatrix}
D& F_{L-1} & \cdots 
& F_{l}
\end{bmatrix}).
\]
\end{lemma}

\begin{proof}
It is easy to check that these properties hold for the base case, \ie, 
$E_L=D$.
Now we demonstrate the properties of $E_l$ for all $l=L-1, \ldots, 1$.

Assume that 
\BEQ\label{e-el-inv-assump}
\begin{bmatrix}
F_{l-}^T & \zeros
\end{bmatrix}E_{l+1}^{-1}
\begin{bmatrix}
I_n \\ \zeros
\end{bmatrix}
=F_{l-}^T\Sigma_{(l+1)+}^{-1}, 
\EEQ
and 
\BEQ\label{e-el-inv-assump2}
\supp(\begin{bmatrix}
F_{l-}^T & \zeros
\end{bmatrix}E_{l+1}^{-1}) = \supp(
F_{l-}^T
\begin{bmatrix}
D & F_{L-1} 
& \cdots &  F_{l+1}
\end{bmatrix}).
\EEQ

Note that the (negative) bottom block in the block diagonal matrix in 
\eqref{e-cholseky-el-elp1}
is equal to
\BEQ\label{e-chol-blockdiag-plrl}
I_{p_lr_l}+F_l^T\Sigma_{(l+1)+}^{-1}F_l \succ 0.
\EEQ
Recall from \S\ref{sec-properties-sparse-structured}, that
this matrix is block diagonal, consisting of $p_l$ blocks,
each of which is of size $r_l \times r_l$.
Let $R_l V_l R_l^T$ be the Cholesky 
factorization of \eqref{e-chol-blockdiag-plrl}.

Using the relation from \eqref{e-cholseky-el-elp1},
we can express the inverse as
\BEAS
E_l^{-1} &=& 
\begin{bmatrix}
I_{n+s_{(l+1)+}} \\
\begin{bmatrix}
-F_{l}^T & \zeros
\end{bmatrix} E_{l+1}^{-1} & I_{p_lr_l}
\end{bmatrix}^T
\begin{bmatrix}
E_{l+1}^{-1} \\
& -(R_l V_l R_l^T)^{-1}
\end{bmatrix}
\begin{bmatrix}
I_{n+s_{(l+1)+}} \\
\begin{bmatrix}
-F_{l}^T & \zeros 
\end{bmatrix} E_{l+1}^{-1} & I_{p_lr_l}
\end{bmatrix}.
\EEAS
Then the matrix $
\begin{bmatrix}
F_{(l-1)-}^T & \zeros
\end{bmatrix}E_{l}^{-1} 
$ 
is equal to
\[
\begin{bmatrix}
F_{(l-1)-}^T &
\begin{bmatrix}
F_{(l-1)-}^T & \zeros
\end{bmatrix} E_{l+1}^{-1}\begin{bmatrix}
-F_{l} \\ \zeros
\end{bmatrix}
\end{bmatrix}
\begin{bmatrix}
E_{l+1}^{-1} \\
\begin{bmatrix}
(R_l V_l R_l^T)^{-1}F_{l}^T & \zeros 
\end{bmatrix} E_{l+1}^{-1} & -(R_l V_l R_l^T)^{-1}
\end{bmatrix},
\]
which simplifies to
\BEQ\label{e-chol-el-inv-row}
\begin{bmatrix}
F_{(l-1)-}^T & \zeros
\end{bmatrix}E_{l+1}^{-1}
\begin{bmatrix}
\left(I_{n+s_{(l+1)+}} -
\begin{bmatrix}
F_l \\ \zeros
\end{bmatrix}
(R_l V_l R_l^T)^{-1}
\begin{bmatrix}
F_{l}^T & \zeros 
\end{bmatrix} E_{l+1}^{-1}
\right)
&
\begin{bmatrix}
F_l \\ \zeros
\end{bmatrix}(R_l V_l R_l^T)^{-1}
\end{bmatrix}.
\EEQ
Combining \eqref{e-chol-el-inv-row}, \eqref{e-el-inv-assump}, and 
SMW~\eqref{e-inverse-woodbury} we get
\BEAS 
\begin{bmatrix}
F_{(l-1)-}^T & \zeros
\end{bmatrix}E_{l}^{-1}
\begin{bmatrix}
I_n \\ \zeros
\end{bmatrix} &=& F_{(l-1)-}^T\left (\Sigma_{(l+1)+}^{-1} - 
\Sigma_{(l+1)+}^{-1}F_l(R_l V_l R_l^T)^{-1}F_l^T \Sigma_{(l+1)+}^{-1}\right) \\
&=& F_{(l-1)-}^T\Sigma_{l+}^{-1}.
\EEAS

The coefficients of matrix
\BEQ\label{e-chol-m3}
\begin{bmatrix}
F_{(l-1)-}^T & \zeros
\end{bmatrix}E_{l}^{-1}
\begin{bmatrix}
\zeros \\ I_{p_lr_l}
\end{bmatrix}
=
F_{(l-1)-}^T\Sigma_{(l+1)+}^{-1} F_l
(R_l V_l R_l^T)^{-1}
= M_3^T
\EEQ
are obtained during the inverse computation,
see \S\ref{sec-rec-smw}.
Furthermore, we have $\supp(M_3^T)=\supp(F_{(l-1)-}^TF_l)$.

Using assumption \eqref{e-el-inv-assump2}, for any $\tilde l \geq l+1$
we have
\[
\supp \left(\begin{bmatrix}
F_{(l-1)-}^T & \zeros
\end{bmatrix}E_{l+1}^{-1}
\begin{bmatrix}
\zeros \\ I_{p_{\tilde l}r_{\tilde l}} \\ \zeros
\end{bmatrix}\right )
=\supp(F_{(l-1)-}^T F_{\tilde l}).
\]
Similarly, it holds
\[
\supp\left (M_3^T
\begin{bmatrix}
F_{l}^T & \zeros 
\end{bmatrix} E_{l+1}^{-1}  \begin{bmatrix}
\zeros \\ I_{p_{\tilde l}r_{\tilde l}} \\ \zeros
\end{bmatrix}\right )
=\supp(F_{(l-1)-}^TF_l F_l^T F_{\tilde l}).
\]

By Lemma \ref{lem-Fls-prod}, for any $l_1 \leq l_2$
product $F_{l_1}^T F_{l_2}$ 
has $p_{l_2}r_{l_1}r_{l_2}$ nonzero entries,
$\cI(F_{l_1}^T F_{l_2}) = \cJ(F_{l_1})$,
and can be computed $O(nr_{l_1}r_{l_2})$.
By Lemma \ref{lem-blockdiag-prod}, 
for any $l'\leq l-1$ and $\tilde l \geq l+1$, we have
$\cI(F_l F_{l}^T) = \cJ(F_l F_{l}^T) \preceq\cI(F_{l'})$.
This implies $\supp(F_{l'}^T F_l F_{l}^T) = \supp(F_{l'}^T)$,
and consequently
$\supp(F_{l'}^T F_l F_{l}^TF_{\tilde l})=\supp(F_{l'}^TF_{\tilde l})$.
Combining this result with \eqref{e-chol-el-inv-row},
for any $\tilde l \geq l+1$ matrix 
\BEQ
\begin{bmatrix}
F_{(l-1)-}^T & \zeros
\end{bmatrix}E_{l}^{-1}
\begin{bmatrix}
\zeros \\ I_{p_{\tilde l}r_{\tilde l}} \\ \zeros
\end{bmatrix}
=
\begin{bmatrix}
F_{(l-1)-}^T & \zeros
\end{bmatrix}E_{l+1}^{-1}
\begin{bmatrix}
\zeros \\ I_{p_{\tilde l}r_{\tilde l}} \\ \zeros
\end{bmatrix}
-
M_3^T
\begin{bmatrix}
F_{l}^T & \zeros 
\end{bmatrix} E_{l+1}^{-1}  \begin{bmatrix}
\zeros \\ I_{p_{\tilde l}r_{\tilde l}} \\ \zeros
\end{bmatrix}, \label{e-chol-simplif-block-el}
\EEQ
 has the sparsity of 
$\supp(F_{(l-1)-}^T F_l F_{l}^TF_{\tilde l}) = \supp(F_{(l-1)-}^TF_{\tilde l})$.
This implies
\[
\supp(\begin{bmatrix}
F_{(l-1)-}^T & \zeros
\end{bmatrix}E_{l}^{-1}) = \supp(
F_{(l-1)-}^T 
\begin{bmatrix}
D& F_{L-1} & \cdots 
& F_{l}
\end{bmatrix}).
\]

The final result follows by induction.

\end{proof}

\paragraph{Cholesky factors.}
Let the Cholesky factorization of a symmetric matrix $E_{l+1}$
be given by 
\[
E_{l+1} = L^{(l+1)}D^{(l+1)}{L^{(l+1)}}^T.
\]
% Note that $E_{l+1}$
% is not positive definite, as diagonal of
% $D^{(l+1)}$ contains positive and negative entries. 

Using~\eqref{e-cholseky-el-elp1} we have
\BEA
E_l 
&=& \begin{bmatrix}
I_{n+s_{(l+1)+}} \\
\begin{bmatrix}
F_{l}^T & \zeros
\end{bmatrix} E_{l+1}^{-1} & I_{p_lr_l}
\end{bmatrix}
\begin{bmatrix}
E_{l+1} \\
& -R_lV_lR_l^T
\end{bmatrix}
\begin{bmatrix}
I_{n+s_{(l+1)+}} \\
\begin{bmatrix}
F_{l}^T & \zeros
\end{bmatrix} E_{l+1}^{-1} & I_{p_lr_l}
\end{bmatrix}^T \nonumber \\
&=& \begin{bmatrix}
L^{(l+1)} \\
\begin{bmatrix}
F_{l}^T & \zeros
\end{bmatrix} E_{l+1}^{-1}L^{(l+1)} & R_l
\end{bmatrix}
\begin{bmatrix}
D^{(l+1)} \\
& -V_l
\end{bmatrix}
\begin{bmatrix}
L^{(l+1)} \\
\begin{bmatrix}
F_{l}^T & \zeros
\end{bmatrix} E_{l+1}^{-1}L^{(l+1)} & R_l
\end{bmatrix}^T.\label{e-cholseky-el-elp1-fin}
\EEA
Note that the matrix $R_l$ is a block diagonal matrix consisting of $p_l$ blocks,
each of which is a lower 
triangular matrix of size $r_l \times r_l$
(\ie, $\cI(R_l)=\cJ(R_l) \preceq \cI(F_l)$),
see \S\ref{sec-properties-sparse-structured}.
Thus from~\eqref{e-cholseky-el-elp1-fin}, Cholesky factors of $E_{l}$  are
\BEQ\label{e-chol-ll-dl}
L^{(l)}=\begin{bmatrix}
L^{(l+1)} \\
\begin{bmatrix}
F_{l}^T & \zeros
\end{bmatrix} (D^{(l+1)}{L^{(l+1)}}^T)^{-1} & R_l
\end{bmatrix},
\qquad
D^{(l)}=
\begin{bmatrix}
D^{(l+1)} \\
& -V_l
\end{bmatrix}.
\EEQ
Then we also have
\[
(L^{(l)})^{-1}=\begin{bmatrix}
(L^{(l+1)})^{-1} \\
-R_l^{-1}\begin{bmatrix}
F_{l}^T & \zeros
\end{bmatrix} E_{l+1}^{-1} & R_l^{-1}
\end{bmatrix}.
\]
Lemma \ref{lem-sparsity-chol-fact} establishes the sparsity pattern of Cholesky
factors, and, in particular, $\supp(L^{(l)})=\supp((L^{(l)})^{-1})$.

\begin{lemma}\label{lem-sparsity-chol-fact}
Let $F$ and $D$ be factors of PSD MLR $\Sigma$,
and $E$ be its expanded matrix.
Then for all $l=L, \ldots, 1$ and
$\tilde l=L, \ldots, l$,
we have
\BEAS
\supp(
F_{\tilde l}^T
\begin{bmatrix}
D & F_{L-1} 
& \cdots &  F_{l+1}
\end{bmatrix})
 &=& 
\supp(
\begin{bmatrix}
\zeros & I_{p_{\tilde l}r_{\tilde l}} & \zeros
\end{bmatrix}(L^{(l+1)})^{-1})\\
&=&\supp(\begin{bmatrix}
\zeros & I_{p_{\tilde l}r_{\tilde l}} & \zeros
\end{bmatrix}L^{(l+1)}).
\EEAS
\end{lemma}
\begin{proof}
Assume for all $\tilde l=L, \ldots, l+1$ 
we have
\BEAS
\supp(
F_{\tilde l}^T
\begin{bmatrix}
D & F_{L-1} 
& \cdots &  F_{l+1}
\end{bmatrix})
 &=& 
\supp(
\begin{bmatrix}
\zeros & I_{p_{\tilde l}r_{\tilde l}} & \zeros
\end{bmatrix}(L^{(l+1)})^{-1})\\
&=&\supp(\begin{bmatrix}
\zeros & I_{p_{\tilde l}r_{\tilde l}} & \zeros
\end{bmatrix}L^{(l+1)}).
\EEAS

Using \eqref{e-chol-ll-dl} it suffices
to show the sparsity of the bottom block of $L^{(l)}$
and $(L^{(l)})^{-1}$ of size
$p_l r_l \times (n + s_{l+})$.  
The assumptions above imply
\[
\supp(
\begin{bmatrix}
I_n & \zeros
\end{bmatrix}(L^{(l+1)})^{-1}) =
\supp(
D
\begin{bmatrix}
D & F_{L-1} 
& \cdots &  F_{l+1}
\end{bmatrix}),
\]
thus since $D^{(l+1)}$ is diagonal we get
\[
\supp(\begin{bmatrix}
F_{l}^T & \zeros
\end{bmatrix} (D^{(l+1)}{L^{(l+1)}}^T)^{-1}) = 
\supp(F_l^T
\begin{bmatrix}
D & F_{L-1} 
& \cdots &  F_{l+1}
\end{bmatrix}).
\]

Combining Lemma \ref{lem-sparsity-E}
with $\supp(R_l^{-1}F_{l}^T) = \supp(F_l^T)$, 
it follows
\[
\supp(R_l^{-1}\begin{bmatrix}
F_{l}^T & \zeros
\end{bmatrix} E_{l+1}^{-1}) =
\supp(F_l^T
\begin{bmatrix}
D & F_{L-1} 
& \cdots &  F_{l+1}
\end{bmatrix}).
\]
Since $\supp(R_l) = \supp(R_l^{-1})\subseteq \supp(F_l^T F_l)$, the following holds
\BEAS 
\supp(F_l^T
\begin{bmatrix}
D & F_{L-1} 
& \cdots &  F_{l}
\end{bmatrix}) &=&
\supp(
\begin{bmatrix}
\zeros & I_{p_lr_l}
\end{bmatrix}L^{(l)}) \\
&=&
\supp(\begin{bmatrix}
\zeros & I_{p_lr_l}
\end{bmatrix}(L^{(l)})^{-1}).
\EEAS
Combining these results with \eqref{e-chol-ll-dl} we conclude  
$\supp(L^{(l)})=\supp((L^{(l)})^{-1})$.

Evidently for the base case, $L^{(L)}=I_n$ and $D^{(L)}=D$, these properties hold.
By induction we showed $\supp(L^{(1)})=\supp((L^{(1)})^{-1})$.
\end{proof}

\subsection{Efficient computation}
\paragraph{Recurrent term.}
Using \eqref{e-chol-ll-dl} we recursively compute
\[
\begin{bmatrix}
F_{(l-1)-}^T & \zeros
\end{bmatrix}(D^{(l)}{L^{(l)}}^T)^{-1}
=
\begin{bmatrix}
F_{(l-1)-}^T & \zeros
\end{bmatrix}\begin{bmatrix}
(D^{(l+1)}{L^{(l+1)}}^T)^{-1} 
& E_{l+1}^{-1}
\begin{bmatrix}
F_l \\ \zeros
\end{bmatrix}(V_l R_l^T)^{-1}
\end{bmatrix}.
\]
Lemma \ref{lem-sparsity-E} implies
\[
\begin{bmatrix}
F_{(l-1)-}^T & \zeros
\end{bmatrix}E_{l+1}^{-1}
\begin{bmatrix}
F_l \\ \zeros
\end{bmatrix}
(V_l R_l^T)^{-1}
% =F_{(l-1)-}^T(F_{(l+1)+}F_{(l+1)+}^T+D)^{-1}F_l(V_l R_l^T)^{-1}
= M_3^TR_l.
\]
The product $M_3^T R_l$ requires 
$\sum_{l'=1}^{l-1}O(p_lr_l^2r_{l'})
= O(p_lr r_l^2)$
operations.
% For all $l = L, \ldots, 1$, 
% the complexity is
% $O(nr + p_{L-1}r^2r_{\max})$.
Thus we get identity
\BEQ\label{e-chol-simplif-block-rec}
\begin{bmatrix}
F_{(l-1)-}^T & \zeros
\end{bmatrix}(D^{(l)}{L^{(l)}}^T)^{-1}
=
\begin{bmatrix}
\begin{bmatrix}
F_{(l-1)-}^T & \zeros
\end{bmatrix}(D^{(l+1)}{L^{(l+1)}}^T)^{-1} 
& M_3^TR_l
\end{bmatrix}.
\EEQ
This indicates that constructing a recurrent term at the next level
only requires computing $M_3^T R_l$.

Moreover, by Lemma \ref{lem-sparsity-chol-fact} the sparsity is
\[
\supp(\begin{bmatrix}
F_{(l-1)-}^T & \zeros
\end{bmatrix}(D^{(l)}{L^{(l)}}^T)^{-1})
=\supp(
F_{(l-1)-}^T\begin{bmatrix}
D & F_{L-1} & \cdots & F_l
\end{bmatrix}).
\]

\paragraph{Method.}
We now describe the algorithm for computing Cholesky factorization,
that recursively computes Cholesky factors of $E_L, E_{L-1}, \ldots,
E_1$.
This process is accompanied by the recursive computation 
of coefficients in $\Sigma^{-1}$,
see \S\ref{s-inverse}.
We include additional time and space complexities 
beyond those discussed in \S\ref{s-inverse}.

We start with $L^{(L)}=I_n$ and $D^{(L)}=D$. Then for each
level $l=L-1, \ldots, 1$ repeat
the following steps.
\begin{enumerate}
\item Compute Cholesky decomposition of \eqref{e-chol-blockdiag-plrl},
$R_lV_lR_l^T$, in $O(p_lr_l^3)$, 
and store its $O(p_lr_l^2)$ coefficients.
The coefficients of \eqref{e-chol-blockdiag-plrl} and
its inverse are obtained in \S\ref{s-inverse}.
\item Form $L^{(l)}$ and $D^{(l)}$ using stored coefficients
of $\begin{bmatrix}
F_{l-}^T & \zeros
\end{bmatrix} (D^{(l+1)}{L^{(l+1)}}^T)^{-1}$
according to \eqref{e-chol-ll-dl}.
\item Form $\begin{bmatrix}
F_{(l-1)-}^T & \zeros
\end{bmatrix}(D^{(l)}{L^{(l)}}^T)^{-1}$ using \eqref{e-chol-simplif-block-rec}.
This requires computing $M_3^TR_l$ with $O(p_lr r_l^2)$ operations and
$\sum_{l'=1}^{l-1}r_{l'}(n + \sum_{\tilde l=l}^{L-1} p_{\tilde l}r_{\tilde l})$ 
coefficients.
We use $\sum_{l'=1}^{l-1} p_lr_{l'}r_l$ 
coefficients of
$M_3$ from \S\ref{s-inverse}.
% \item Proceed to the next level $l-1$.
\end{enumerate}

Cholesky factor of $E$, lower triangular matrix $L^{(1)}$, 
has less than
\[
n + \sum_{l=L-1}^1 r_l(n + p_{L-1}r_{L-1}+\cdots+p_l r_l)
\leq
nr + p_{L-1} r^2
\]
nonzero entries.
The total cost for computing the factors is
\[
O(nr) + \sum_{l=L-1}^1 O(p_lr_l^3 + p_lr r_l^2) = 
O(nr + p_{L-1}r^3).
\]

\subsection{Determinant}\label{sec-chol-determinant}
Using the Cholesky decomposition of $E$ we can easily compute 
the determinant of MLR covariance matrix $\Sigma$.
Specifically, using \eqref{e-chol-ldu} we have
\[
\det (E) = \det(FF^T+D)(-1)^s,
\]
since the eigenvalues of a triangular matrix are exactly its diagonal
entries and because the determinant is a multiplicative map. 
Alternatively, using Cholesky decomposition, $E=L^{(1)}D^{(1)}{L^{(1)}}^T$,
we have
\[
\det (E) = \det(L^{(1)})^2 \det(D^{(1)}) = \det(D^{(1)}).
\]
Since 
\[
\det(D^{(1)}) = (-1)^s\det(D)\prod_{l=1}^{L-1} \det(V_l),
\]
we obtain 
\[
\det(FF^T+D) = \prod_{i=1}^{n+s}|D^{(1)}_{ii}|, 
\qquad 
\log \det(FF^T+D) = \sum_{i=1}^{n+s}\log |D^{(1)}_{ii}|.
\]

\begin{remark}\label{rem-determinant}
The $\det(\Sigma)$ can be computed at no additional cost while recursively computing 
the coefficients in $\Sigma^{-1}$, see \S\ref{s-inverse}.
For every $l=L-1, \ldots, 1$ we compute the eigendecomposition of the matrix
\[
R_l V_l R_l^T = I_{p_lr_l}+F_l^T\Sigma_{(l+1)+}^{-1}F_l = Q_l \Lambda_l Q_l^T,
\]
which implies
\[
\det(V_l) = \det(R_l V_l R_l^T) = \det(Q_l \Lambda_l Q_l^T) = \det(\Lambda_l).
\]
Therefore, 
\BEQ\label{e-determ-rec-lemma}
\det(FF^T+D) = \det(D)\prod_{l=1}^{L-1} \det(\Lambda_l).
\EEQ
\end{remark}

Note that, alternatively, determinant \eqref{e-determ-rec-lemma}
can be interpreted as relying on the recursive application of the matrix determinant lemma, 
which states that if $A\in \reals^{n \times n}$ is invertible, then
for any $U, V \in \reals^{n \times p}$, it holds
\[
\det(A + UV^T) = \det(A) \det(I_p + V^T A^{-1} U).
\]

\section{Factor model with linear covariates}\label{sec-lin-cov-fm}

In this section we show how to apply our fitting method to the factor model with linear covariates.
Suppose we have samples $y_1, \ldots, y_N \in \reals^n$
along with covariates $x_1, \ldots, x_N \in \reals^{p}$.
Then the factor model with the covariates is
given by
\[
y_i = B x_i + F z_i + e_i,
\]
where $B \in \reals^{n \times p}$ is a matrix with regression 
coefficients.
Define 
\[
X = \left[ \begin{array}{c} 
x_1^T \\ \vdots \\ x_N^T
\end{array} \right]  \in \reals^{N \times p},
\quad 
\tilde Z = \left[\begin{array}{cc}
X & Z \end{array}\right] \in \reals^{N \times (p+s)} ,
\quad 
\tilde F = \left[\begin{array}{cc}
B & F \end{array}\right] \in \reals^{n \times (p+s)}.
\]
Similarly to steps in \S\ref{sec-e-step}, we have 
$y_i \sim \mathcal{N}(B x_i, \Sigma)$, $z_i \sim \mathcal{N}(0, I_s)$,
% $(z, y)$ is a Gaussian random vector with zero mean and covariance
% \[
% \Cov\left( (z, y), (z, y) \right) = \left[ \begin{array}{cc} 
% I_{s} & F^T \\ 
% F & \Sigma
% \end{array} \right],
% \]
and the conditional distribution $(z_i \mid y_i, x_i, \tilde F^0, D^0)$ is Gaussian, 
\[
 \mathcal N\left( {F^0}^T(\Sigma^0)^{-1}(y_i - B^0 x_i ), I_{s} - {F^0}^T(\Sigma^0)^{-1}F^0 \right).
\]

Since the log-likelihood of complete data $(Y, X, Z)$ is
\[
\ell(\tilde F, D; Y, \tilde Z) 
= -  \frac{(n+s)N}{2}\log(2\pi) -\frac{N}{2}\log \det D - 
\frac{1}{2} \| (Y- \tilde Z \tilde F^T)  D^{-1/2}\|_F^2 - \frac{1}{2} \|Z\|_F^2,
\]
 we have
\BEAS
Q(\tilde F, D; \tilde F^0, D^0) &=& \Expect \left ( \ell(\tilde F, D; Y, \tilde Z)  \mid Y, X, \tilde F^0, D^0  \right ) \\
% &=& -  \frac{(n+s)N}{2}\log(2\pi) -\frac{N}{2}\log \det D  \\ 
% && - \frac{1}{2} 
% \sum_{i=1}^N \Expect \Big ( \Tr (D^{-1}(y_i - \tilde F \tilde z_i)(y_i - \tilde F \tilde z_i)^T)\\
% && + \Tr (z_i z_i^T) \mid Y, X, \tilde F^0, D^0  \Big ) \\
% &=& -  \frac{(n+s)N}{2}\log(2\pi) -\frac{N}{2}\log \det D \\
% &&- \frac{1}{2} 
% \sum_{i=1}^N \Tr \big(\Expect \left (z_iz_i^T \mid y_i, x_i, \tilde F^0, D^0  \right ) \big)\\
% &&- \frac{1}{2} 
% \sum_{i=1}^N \Tr  \big(D^{-1} \{  
% (y_i y_i^T - 2\tilde F \Expect \left ( \tilde z_i \mid y_i, x_i, \tilde F^0, D^0  \right ) y_i^T)  \\
% &&+  \tilde F \Expect \left (\tilde z_i \tilde z_i^T \mid y_i, x_i, \tilde F^0, D^0  \right ) \tilde F^T
% \} \big)\\
&=& -  \frac{(n+s)N}{2}\log(2\pi) -\frac{N}{2}\log \det D 
- \frac{1}{2} \Tr (\tilde W) \\
&&- \frac{1}{2} 
\Tr  \bigg(D^{-1} \bigg \{  
Y^TY - 2 \tilde F \left[\begin{array}{cc}
X^T Y \\ \tilde V Y \end{array}\right]  
+ \tilde F \left[\begin{array}{cc}
X^T X & X^T \tilde V^T\\ \tilde V X & \tilde W \end{array}\right]  \tilde F^T \bigg \}
\bigg),
\EEAS
where the matrices $\tilde V$ and $\tilde W$ are defined as
\BEAS 
\tilde V &=& {F^0}^T(\Sigma^0)^{-1}(Y - X{B^0}^T )^T \\
\tilde W &=& \sum_{i=1}^N\Expect \left ( z_i z_i^T \mid y_i, x_i, \tilde F^0, D^0 \right ) = 
N(I_{s} - 
{F^0}^T(\Sigma^0)^{-1}F^0) + \tilde V \tilde V^T.
\EEAS

Similarly to \eqref{e-expect-step}, $Q(\tilde F, D; \tilde F^0, D^0)$
is separable across the rows of $\tilde F$, therefore, our fast EM method 
can be applied directly.

\section{Product of MLR matrices}\label{sec-mlr-matmul}
In this section we show that the product of two MLR matrices, 
$A$ with MLR-rank $r$ and $A'$
with MLR-rank $r'$, sharing the same symmetric 
hierarchical partition, 
is also an MLR matrix with the same hierarchical partition and an 
MLR-rank of $(r + r')$. 
We also show that it can be computed using 
$O(n\max\{r, r'\}^2)$ operations.

Since the hierarchical partition is symmetric,
without loss of generality assume $A$ and $A'$ are contiguous MLR. 
Define 
\[
A_{l+} = A_l + \cdots + A_L,
\]
then it is easy to check that
\BEA\label{e-mlr-matmul-problem}
A A' &=& \left (\sum_{l=1}^L A_l\right)\left 
(\sum_{l=1}^L A'_l\right) \nonumber \\
&=& \sum_{l=1}^{L-1} \left(A_l A'_{l+} + 
A_{(l+1)+} A'_l \right) + A_L A'_L.
\EEA 
We now show that each term in the sum above can be decomposed 
into a product of block diagonal matrices, which are the factors of
matrix $A A'$ on level $l$.

Recall the notation from~\cite{parshakova2023factor},
\[
A_l = \blkdiag(B_{l,1}C_{l,1}^T, \ldots, B_{l,p_l}C_{l,p_l}^T), \qquad
A'_l = \blkdiag(B'_{l,1}C_{l,1}'^T, \ldots, B'_{l,p_l}C_{l,p_l}'^T)
\]
where $B_{l,k}, B'_{l,k}, C_{l,k}, C'_{l,k} \in \reals^{n_{l,k}\times r_l}$,
for all $k=1, \ldots, p_l$, and $l=1, \ldots, L$.

Since for all levels $l \leq \tilde l$,
$\supp(A'_{\tilde l})\subseteq \supp(A_{l})$,
it follows that $\supp(A_{l}A'_{\tilde l})= \supp(A_{l})$, see \S\ref{sec-properties-sparse-structured}.
Thus we also have
$\supp(A_l A'_{l+}) = \supp(A_l)$.
Similarly, $\supp(A_{(l+1)+} A'_l) = \supp(A'_l)$.

Consider levels $l \leq \tilde l$.
Let the $k$th group
on level $l$ (for $k=1, \ldots, p_{l}$)
be partitioned into $p_{l, k, \tilde l}$ groups on level $\tilde l$,
indexed by $\tilde k, \ldots, \tilde k + p_{l, k, \tilde l} - 1$. 
Let the partition of $C_{l,k}$ into $p_{l, k, \tilde l}$ blocks for 
each group be defined as follows
\[
C_{l,k} = \left[ \begin{array}{c} 
C_{l,k,1} \\ \vdots \\ C_{l,k,p_{l, k, \tilde l}}
\end{array} \right].
\]
Then the $k$th diagonal block of the $A_l A'_{\tilde l}$
is given by
\BEAS
(A_l A'_{\tilde l})_k &=& 
B_{l,k}C_{l,k}^T
\blkdiag\left(B'_{\tilde l,\tilde k}C_{\tilde l,
\tilde k}'^T, \ldots, 
B'_{\tilde l,p_{l, k, \tilde l}}C_{\tilde l, 
p_{l, k, \tilde l}}'^T\right)\\
&=&
B_{l,k}\left[ \begin{array}{ccc} 
 C_{l,k, 1}^TB'_{\tilde l,\tilde k}C_{\tilde l,\tilde k}'^T & \cdots & 
 C_{l,k, p_{l, k, \tilde l}}^T
B'_{\tilde l,\tilde k+p_{l, k, \tilde l}-1}C_{\tilde l,
\tilde k+p_{l, k, \tilde l}-1}'^T
\end{array}\right]\\
&=& B_{l,k} \overline C_{l,k, \tilde l}^T,
\EEAS
where $B_{l,k}, \overline C_{l,k, \tilde l} \in \reals^{n_{l,k} \times r_l}$ 
are left and right factors of $(A_l A'_{\tilde l})_k$.
Computing 
\[
(C_{l,k, j}^TB'_{\tilde l,j})C_{\tilde l,\tilde k+j-1}'^T \in 
\reals^{r_l \times n_{\tilde l, \tilde k+j-1}},
\quad j = 1, \ldots, p_{l, k, \tilde l},
\]
where $C_{l,k, j} \in \reals^{n_{\tilde l, \tilde k+j-1}\times r_l}$ and 
$C'_{\tilde l, \tilde k+j-1}, B'_{\tilde l, \tilde k+j-1} \in 
\reals^{n_{\tilde l, \tilde k+j-1}\times r_{\tilde l}}$,
takes $O(n_{\tilde l, \tilde k+j-1}r_l r_{\tilde l})$ operations.
Computing all coefficients of the right factor of $A_l A'_{\tilde l}$
requires
\[
\sum_{\tilde k=1}^{p_{\tilde l}}
\sum_{j=1}^{p_{l, k, \tilde l}}O(n_{\tilde l, \tilde k+j-1} r_l r_{\tilde l})  
= O(nr_l r_{\tilde l}).
\]
Therefore, we have the following factorization
\[
A_l A'_{\tilde l} = \blkdiag(B_{l,1}\overline C_{l,1, \tilde l}^T, 
\ldots, B_{l,p_l} 
\overline C_{l,p_l, \tilde l}^T) = B_l \overline C_{l, \tilde l}^T,
\]
where $\supp(\overline C_{l, \tilde l}) =\supp(B_l)$.

Similarly, for levels $l \geq \tilde l$, we have
\[
A_l A'_{\tilde l} = \blkdiag(\overline B_{\tilde l,1,l} 
C_{\tilde l,1}'^T, \ldots, 
\overline B_{\tilde l,p_{\tilde l}, l} C_{\tilde l,p_{\tilde l}}'^T) = 
\overline B_{\tilde l, l} C_{\tilde l}'^T,
\]
where $\supp(\overline B_{\tilde l, l})=\supp(C_{\tilde l})$,
and it can be computed
in $O(nr_l r_{\tilde l})$.
Thus $A_l A'_{\tilde l}$ has the same sparsity as $A'_{\tilde l}$.

Combining the above we have the following factorization
\BEA\label{e-mlr-matmul-level}
A_l A'_{l+} + A_{(l+1)+} A'_l &=& 
\sum_{\tilde l = l}^L B_l \overline C_{l, \tilde l}^T
+
\sum_{\tilde l = l+1}^L \overline B_{l, \tilde l} C_{l}'^T \nonumber \\
&=& 
\left[ \begin{array}{ccc}  B_l & \sum_{\tilde l = l+1}^L 
\overline B_{l, \tilde l}
\end{array}\right]
\left[ \begin{array}{ccc} 
\sum_{\tilde l = l}^L \overline C_{l, \tilde l} & C'_l
\end{array}\right]^T,
\EEA
which we can compute in 
\[
O\left (nr_l \sum_{\tilde l=l+1}^L r'_{\tilde l} + 
nr'_l \sum_{\tilde l=l}^L r_{\tilde l} \right).
\]
Note that $\cI(\sum_{\tilde l = l+1}^L \overline B_{l, \tilde l})=\cI(B_l)$, and similarly, 
$\cI(\sum_{\tilde l = l}^L \overline C_{l, \tilde l})=\cI(C'_l)$.
Therefore, we can equivalently represent~\eqref{e-mlr-matmul-level}
as a product of two block diagonal matrices by permuting the columns 
in the left and right factors accordingly.
The resulting two block diagonal matrices are the factors of $AA'$ on level $l$, 
and in the compressed form have size $n\times(r_l + r'_l)$  each.

Finally, from~\eqref{e-mlr-matmul-problem} we see that matrix
$AA'$ is an MLR matrix with MLR-rank $(r+r')$. Moreover, computing factors 
requires $O(n\max\{r, r'\}^2)$ operations.

\end{document}

%% file: defs.tex
\newcommand{\ones}{\mathbf 1}
\newcommand{\zeros}{\mathbf 0}

\newcommand{\reals}{{\mbox{\bf R}}}

\newcommand{\symm}{{\mbox{\bf S}}}  % symmetric matrices
\newcommand{\BEAS}{\begin{eqnarray*}}
\newcommand{\EEAS}{\end{eqnarray*}}
\newcommand{\BEA}{\begin{eqnarray}}
\newcommand{\EEA}{\end{eqnarray}}
\newcommand{\BEQ}{\begin{equation}}
\newcommand{\EEQ}{\end{equation}}
\newcommand{\BIT}{\begin{itemize}}
\newcommand{\EIT}{\end{itemize}}

\newcommand{\Tr}{\mathop{\bf Tr}}
\newcommand{\blkdiag}{\mathop{\bf blkdiag}}

\newcommand{\Expect}{\mathop{\bf E{}}}

\newcommand{\Var}{\mathop{\bf var}}
\newcommand{\Cov}{\mathop{\bf cov}}
\newcommand{\dom}{\mathop{\bf dom}} % domain

\newcommand{\cI}{\mathcal{I}}
\newcommand{\cJ}{\mathcal{J}}

\newcommand{\eg}{{\it e.g.}}
\newcommand{\ie}{{\it i.e.}}

\newcounter{algorithmctr}[section]
\renewcommand{\thealgorithmctr}{\thesection.\arabic{algorithmctr}}
   {\refstepcounter{algorithmctr}
   \begin{list}{}{%
       \setlength{\rightmargin}{0\linewidth}%
       \setlength{\leftmargin}{.05\linewidth}}%
       \rmfamily\small
       \item[]{\setlength{\parskip}{0ex}\hrulefill\par%
        \nopagebreak{\bfseries\textsf{Algorithm \thealgorithmctr~}}}}%
   {{\setlength{\parskip}{-1ex}\nopagebreak\par\hrulefill} 
   \end{list}}

   {\refstepcounter{algorithmctr}
   \begin{list}{}{%
       \setlength{\rightmargin}{0\linewidth}%
       \setlength{\leftmargin}{.05\linewidth}}%
       \rmfamily\small
       \item[]{\setlength{\parskip}{0ex}\hrulefill\par%
        \nopagebreak{\bfseries\textsf{Algorithm}}}}%
   {{\setlength{\parskip}{-1ex}\nopagebreak\par\hrulefill} 
   \end{list}}

\newcommand{\supp}{\mathop{\bf supp}}